\documentclass[pdflatex,sn-mathphys]{sn-jnl}

\jyear{2022}%

\theoremstyle{thmstyleone}%
\newtheorem{theorem}{Theorem}
\theoremstyle{thmstyletwo}%
\newtheorem{remark}{Remark}%

\theoremstyle{thmstylethree}%
\newtheorem{definition}{Definition}%
\usepackage{wrapfig}
\usepackage{enumerate}
\usepackage[utf8]{inputenc} 
\usepackage[T1]{fontenc}    
\usepackage{hyperref}       
\usepackage{url}            
\usepackage{booktabs}       
\usepackage{amsfonts}       
\usepackage{verbatim}       
\usepackage{nicefrac}       
\usepackage{microtype}      
\usepackage{scalerel}

\usepackage{times}  
\usepackage{helvet} 
\usepackage{courier}  
\usepackage{graphicx} 
\usepackage{lipsum}
\usepackage{multicol}

\newcommand{\ignore}[1]{}

\usepackage{epsfig}
\usepackage{amsmath,amsthm,amssymb}
\usepackage{amssymb}

\usepackage{mathtools}
\usepackage{bm}
\usepackage{xcolor}
\usepackage{amsthm}
\newtheorem{lemma}{Lemma}
\usepackage{multirow}
\usepackage{fancyvrb}

\newcounter{relctr} 
\everydisplay\expandafter{\the\everydisplay\setcounter{relctr}{0}} 

\newcommand\labelrel[2]{%
  \begingroup
    \refstepcounter{relctr}%
    \stackrel{\textnormal{(\alph{relctr})}}{\mathstrut{#1}}%
    \originallabel{#2}%
  \endgroup
}
\AtBeginDocument{\let\originallabel\label} 

\graphicspath{{figs/}}

\newcommand\mydots{\ifmmode\ldots\else\makebox[1em][c]{.\hfil.\hfil.}\fi}
\newcommand{\W}{{\mathcal{W}}}
\newcommand{\WU}[1]{\W^{#1}}
\newcommand{\Wdist}[1]{{\W^{#1}}}

\newcommand{\Dist}{{d}}
\newcommand{\DistU}[1]{\Dist^{#1}}
\newcommand{\DistUL}[2]{{\DistU{#1}_{#2}}}
\newcommand{\NumDist}{n}

\newcommand{\X}{{\Omega}}
\newcommand{\XU}[1]{\X^{#1}}
\newcommand{\XUL}[2]{\XU{#1}_{#2}}

\newcommand{\SizeDist}{\vert \Omega \vert }
\newcommand{\SizeDistU}[1]{{\vert \Omega ^{#1}\vert}}
\newcommand{\optsymbol}{*}

\newcommand{\PI}{{\bm p}}
\newcommand{\PIU}[1]{\PI^{#1}}

\newcommand{\PIUL}[2]{\PIU{#1}_{#2}}
\newcommand{\PIs}{{\PI}^{\optsymbol}}
\newcommand{\PIsU}[1]{{\PIs}^{#1}}
\newcommand{\PIsUL}[2]{\PIsU{#1}_{#2}}
\newcommand{\QI}{{\bm q}}
\newcommand{\QIU}[1]{\QI^{#1}}
\newcommand{\QIUL}[2]{\QIU{#1}_{#2}}
\newcommand{\RI}{{\bm r}}
\newcommand{\RIU}[1]{\RI^{#1}}
\newcommand{\RIUL}[2]{\RIU{#1}_{#2}}
\newcommand{\RIs}{{\RI}^{\optsymbol}}
\newcommand{\RIsU}[1]{{\RIs}^{#1}}
\newcommand{\RIsUL}[2]{\RIsU{#1}_{#2}}
\newcommand{\G}{{\mathcal{G}}}
\newcommand{\Gmap}[1]{{\G\left(#1 \right)}}
\newcommand{\HH}{{\mathcal{H}}}
\newcommand{\HU}[1]{\HH^{#1}}
\newcommand{\HUL}[2]{\HU{#1}_{#2}}
\newcommand{\inner}[2]{\left \langle #1,#2 \right \rangle}

\newcommand{\inlineitem}[1][]{%
\ifnum\enit@type=\tw@
    {\descriptionlabel{#1}}
  \hspace{\labelsep}%
\else
  \ifnum\enit@type=\z@
       \refstepcounter{\@listctr}\fi
    \quad\@itemlabel\hspace{\labelsep}%
\fi}

\begin{document}

\title[A Family of Pairwise MMOT that Define a Generalized Metric]{A Family of Pairwise Multi-Marginal Optimal
Transports that Define a Generalized Metric}

\author[1]{\fnm{Liang} \sur{Mi}}\email{icemiliang@gmail.com}

\author[2]{\fnm{Azadeh} \sur{Sheikholeslami}}\email{asheikholeslami@suffolk.edu}

\author*[3]{\fnm{Jos\'e} \sur{Bento}}\email{jose.bento@bc.edu}

\affil[1]{\orgdiv{Imaging and Computer Vision Research}, \orgname{Apple Inc.}, \orgaddress{\street{1 Infinite Loop}, \city{Cupertino}, \postcode{95014}, \state{CA}, \country{USA}}}

\affil[2]{\orgdiv{Computer Science Department}, \orgname{Suffolk University}, \orgaddress{\street{73 Tremont St}, \city{Boston}, \postcode{02108}, \state{MA}, \country{USA}}}

\affil*[3]{\orgdiv{Computer Science Department}, \orgname{Boston College}, \orgaddress{\street{140 Commonwealth Ave.}, \city{Chestnut Hill}, \postcode{02467}, \state{MA}, \country{USA}}}

\abstract{The Optimal transport (OT) problem is rapidly finding its way into machine learning. Favoring its use are its metric properties. Many problems admit solutions with guarantees only for objects embedded in metric spaces, and the use of non-metrics can complicate  solving them. Multi-marginal OT (MMOT) generalizes OT to simultaneously transporting multiple distributions. It captures important relations that are missed if the transport only involves two distributions. Research on MMOT, however, has been focused on its existence, uniqueness, practical algorithms, and the choice of cost functions. There is a lack of discussion on the metric properties of MMOT, which limits its theoretical and practical use. Here, we prove new generalized metric properties for a  family of pairwise MMOTs. We first explain the difficulty of proving this via two negative results. Afterward, we prove the
MMOTs' metric properties. Finally, we show that the generalized triangle inequality of this family of MMOTs cannot be improved. We illustrate the superiority of our MMOTs over other generalized metrics, and over non-metrics in both synthetic and real tasks.}


\maketitle

\section{Introduction}\label{sec1}

 Let $(\XU{1}{,\mathcal{F}^1},\PIU{1})$ and $(\XU{2}{,\mathcal{F}^2},\PIU{2})$ be two  probability spaces. Given a cost function  $\DistU{}: \XU{1} \times \XU{2} \rightarrow \mathbb{R}^{\geq 0}$, and $\ell \geq 1$, the (Kantorovich) \emph{Optimal Transport} (OT) problem  \cite{kantorovich1942translocation} seeks:
%
%
\begin{equation}\label{eq:ot}
    \inf_{\PIU{}} \bigg(\;{\scaleobj{.9}{\int}}_{\XU{1} \times \XU{2}} \Dist^{\ell} {\rm d}\PIU{}\bigg)^{\frac{1}{\ell}} 
    \text{, subject to: } 
    {\scaleobj{.9}{\int}}_{\XU{1}} {\rm d}\PIU{} = \PIU{2} \text{ and } {\scaleobj{.9}{\int}}_{\XU{2}} {\rm d}\PIU{} = \PIU{1},
\end{equation}
%
where the $\inf$ is over measures $\PIU{}$ on $\XU{1} \times \XU{2}$. 
Problem \eqref{eq:ot} is typically studied under the assumptions that $\XU{1}$ and $\XU{2}$ are in a Polish space on which $d$ is a metric, in which case the minimum of \eqref{eq:ot} is the \emph{Wasserstein distance} (WD). The WD is popular in many applications including shape interpolation~\cite{solomon2015convolutional}, generative modeling~\cite{arjovsky2017wasserstein,fan2017point}, domain adaptation~\cite{damodaran2018deepjdot}, and dictionary learning~\cite{schmitz2018wasserstein}.
The WD is a metric on the space of probability measures \cite{ambrosio2013user}, and this property is useful in many ML tasks, e.g., clustering \cite{xing2003distance,hartigan1975clustering}, nearest-neighbor search \cite{clarkson2006nearest,clarkson1999nearest,beygelzimer2006cover}, and outlier detection \cite{angiulli2002fast}. Indeed, some of these tasks are tractable, or allow theoretical guarantees, when done on a metric space. E.g., finding the nearest neighbor \cite{clarkson2006nearest,clarkson1999nearest,beygelzimer2006cover} or the diameter \cite{indyk1999sublinear} of a dataset requires a polylogarithimic computational effort under metric assumptions;  approximation algorithms for clustering rely on metric assumptions, whose absence worsens known bounds \cite{ackermann2010clustering}; also, \cite{memoli2011gromov} uses the metric properties of the WD to study object matching via metric invariants.

Recently, a generalization of OT to multiple marginal measures has gained attention. Given probability spaces $(\XU{i}{,\mathcal{F}^i},\PIU{i})$, $ i  = 1,\mydots,\NumDist$, a function $d:\XU{1}\times\mydots\times\XU{\NumDist}\mapsto \mathbb{R}^{\geq 0}$, a number $\ell \geq 1$, and $\XU{ \backslash i}\triangleq\XU{1}\times\mydots\times\XU{i-1}\times \XU{i+1}\times\mydots\times \XU{\NumDist}$, Multi-Marginal Optimal Transport (MMOT) seeks:
%
\begin{align}\label{eq:multi_ot}
    \inf_{\PIU{}} \left({\scaleobj{.9}{\int}}_{\XU{1} \times\mydots\times \XU{\NumDist}} \hspace{0cm} \Dist^{\ell} {\rm d}\PIU{} \right)^{\frac{1}{\ell}}
    \text{, subject to: }
    {\scaleobj{.9}{\int}}_{\XU{\backslash i}} \hspace{0cm} {\rm d}\PIU{} = \PIU{i},  i  = 1,\mydots,\NumDist,
\end{align}
%
 where the infimum is taken over measures $\PIU{}$ on $\XU{1} \times\mydots\times \XU{\NumDist}$. The term MMOT was coined in~\cite{pass2012local}, and was surveyed by the same authors in \cite{pass2015multi}.

Unfortunately, there is a lack of discussion about the (generalized) metric properties of MMOT. 
Much of the discussion on MMOT has focused on the existence of a minimizer, the uniqueness and structure of
both Monge and Kantorovich solutions, applications, practical algorithms, and the choice of the cost function \cite{pass2012multi,peyre2019computational,gerolin2019duality,moameni2017solutions}.
Since the metric property of the WD is useful in  many applications, understanding when the (potential) minimum of \eqref{eq:multi_ot}, $\W(\PIU{1},\mydots,\PIU{\NumDist})$, a \emph{multi-way distance}, has metric-like properties is critical.
For example, metric properties can improve distance-based clustering,
so too can generalized metrics improve clustering based on multi-way distances.
A precise definition of generalized metrics and of the generalized triangle inequality that these must satisfy is presented in Section \ref{sec:def}.
Later, in Section \ref{sec:num_exp}, we illustrate the improvement in using metrics on a task of clustering chemical compounds.
Importantly, several algorithms in
\cite{xing2003distance, hartigan1975clustering, clarkson2006nearest, clarkson1999nearest, beygelzimer2006cover, angiulli2002fast, indyk1999sublinear, ackermann2010clustering}, and more,
which use distances including WD as input, have guarantees if the distances are metrics. They extend to feeding off multi-distances, and hence can use MMOT, and have guarantees under generalized metrics similar to those under classic metrics. We now exemplify these extensions, and their potential applications:



{\bf Example 1:} Given a set $S$ with $\NumDist$ distinct distributions we can find its $3$-diameter $\Delta \triangleq \max_{\PIU{1},\PIU{2},\PIU{3} \in S} \W(\PIU{1},\PIU{2},\PIU{3})$ with  $\NumDist \choose 3$ evaluations of $\W$.
What if $\W$ satisfies the generalized triangle inequality  $\W^{1,2,3}\leq \W^{4,2,3} + \W^{1,4,3} + \W^{1,2,4}$, where $\W^{i,j,k} \triangleq \W(\PIU{i},\PIU{j},\PIU{k})$?
Let $\Delta = \W(\PIsU{1},\PIsU{2},\PIsU{3})$ for some $\PIsU{1},\PIsU{2},\PIsU{3}$ and let us extend the notation $\W^{i,j,k}$ such that e.g. $\W^{i,*j,k} \triangleq \W(\PIU{i},\PIsU{j},\PIU{k})$.
We can now show that for at least $\NumDist/3$ distribution triplets $\W \geq \Delta/3$. 
Indeed, for all $\PIU{4} \in S$, we cannot simultaneously have $\W^{4,*2,*3},\W^{*1,*2,4},\W^{*1,4,*3}< \Delta/3$. Therefore, if we evaluate $\W$ on random distribution triplets, the probability of finding a triple for which $\W > \Delta/3$ is at least $\Omega(1/\NumDist^2)$\footnote{The asymptotic notation used in this paper is as follows:
$f(n)=\mathcal{O}(g(n))$ denotes an asymptotically tight upper-bound on $f(n)$, i.e.,  there exist $M, n_0>0$ such that $0 \leq f(n) \leq M g(n)$ for all $n\geq n_0$;
$f(n)=\Omega(g(n))$ denotes an asymptotic lower-bound on $f(n)$, i.e., there exists $m, n_0>0$ such that $0 \leq m g(n)\leq f(n)$ for all $n\geq n_0$; $f(n)=\Theta(g(n))$ denotes an asymptotically tight bound on $f(n)$, i.e., there exist $m, M, n_0>0$,  such that $0 \leq m g(n)  \leq f(n)\leq M g(n)$ for all $n\geq n_0$. This notation is standard and can be found, e.g., in \cite[Chapter 3]{cormen2009introduction}.}. 
Hence,
we are guaranteed to find a approximation $\Delta'$ of $\Delta$ that satisfies $\Delta \geq \Delta' > \Delta/3$ (also called a $(1/3)$-approximation) with only $\mathcal{O}(\NumDist^2)$ evaluations of $\W$ on average, an improvement over $\NumDist \choose 3$. Diameter estimation relates to outlier detection \cite{angiulli2002fast} which is critical e.g. in cybersecurity \cite{singh2012outlier}.

{\bf Example 2:} Let $S$ be as above. We can compute the average $A \triangleq   \sum_{\PIU{1},\PIU{2},\PIU{3} \in S} \Wdist{}({\PIU{1},\PIU{2},\PIU{3}})/{ n \choose 3}$ with $n \choose 3$ evaluations of $\Wdist{}$. We can estimate $A$ by averaging $\W$ over a set with $\mathcal{O}(n^2)$ distinct triplets randomly sampled from $S$,  improving over $n \choose 3$.

If $\Wdist{}$ is a generalized metric, an argument as in Example 1 shows that with high probability we do not miss triplets with large $\Wdist{}$, which is the critical step to prove that we approximate $A$ well.
Average estimation is critical e.g. in differential privacy \cite{dwork2009differential}.

{\bf Example 3:} Let $S$ be as above. Consider building an hypergraph with nodes $S$ and hyperedges defined as follows. For each distinct triplet $(\PIU{1},\PIU{2},\PIU{3})$ for which $\Wdist{}(\PIU{1},\PIU{2},\PIU{3}) < thr$, a constant threshold, include it as an hyperedge. Hypergraphs are increasingly important in modern ML, specially for clustering using multiwise relationships among objects \cite{ghoshdastidar2015provable, purkait2016clustering}. Let $\Wdist{}$ satisfy the triangle inequality in Example 1 and be invariant under arguments permutations. Let us again define $\Wdist{}(\PIU{i},\PIU{j},\PIU{k}) = \Wdist{i,j,k}$.
One can prove that, for any $\PIU{1},\PIU{2},\PIU{3}$, and $\PIU{4}$, $\Wdist{1,2,3} \geq \max\{ \Wdist{1,2,4}-\Wdist{1,3,4}-\Wdist{2,3,4} , \Wdist{2,3,4}-\Wdist{1,3,4}-\Wdist{1,2,4}, \Wdist{1,3,4}- \Wdist{1,2,4}-\Wdist{2,3,4}\}$. We can use this inequality to quickly ignore triplets with large $\Wdist{1,2,3}$ without evaluating them: plug into it already computed values for $\Wdist{1,2,4}$, $\Wdist{1,3,4}$, $\Wdist{2,3,4}$, do this for multiple $\PIU{4}$ to try to maximize the r.h.s, and check if the r.h.s. is larger than $thr$.

In this paper, we show that an important family of pairwise MMOTs defines a generalized metric. 
The idea and advantages of using pairwise-type MMOTs is not new. Previous work \cite{altschuler2020polynomial, altschuler2021hardness, benamou2015iterative, benamou2016numerical, fan2021complexity, haasler2021multi, haasler2021multimarginal}, however, uses them to make a point about concepts other than their metric properties. Most often, the focus is on the polynomial computability of MMOT under extra assumptions. To the best of our knowledge, we are the first to study the generalized metric properties of families of pairwise MMOT.
There are many applications of MMOT, and of pairwise-type MMOT more specifically. 
In addition to the applications referenced in the aforementioned papers, the survey \cite{pass2015multi} lists applications in economics, physics, and financial mathematics.
Specific examples of applications that use pairwise MMOT and that may benefit from the generalized metric properties we prove are locality sensitive hashing, labeling problem in classification, image registration, multi-agent matching \cite{li2019pairwise}, and optimal coupling of multiple random variables \cite{angel2019pairwise}.

The rest of this paper is organized as follows. We first explain the difficulty of proving this via two negative results (Sections \ref{sec:no_gluing} and \ref{sec:d_n_metric_not_enough}). We present our main results on the generalized metric properties of this family of pairwise MMOT's in Section \ref{sec:MMOT_metric_properties_main}. We present a proof for a simple case to illustrate the key ideas of the main proof in Section \ref{sec:proof_n_3_ell_1}. Finally, with various numerical examples we show that a MMOT from our family, and which defines an n-metric, improves the task of clustering graphs.
%
\section{Definitions and setup}\label{sec:def}
\textbf{Lists.} We write $s_1,\mydots,s_k$ as $s_{1:k}$, $\XU{1},\mydots,\XU{k}$ as $\XU{1:k}$, and $A_{s_1,\mydots,s_k}$ as $A_{s_{1:k}}$. 
Note that $A_{s_{1}:s_{k}}$ differs from $A_{s_{1:k}}$. Assuming $s_k > s_1$, then we have $A_{s_{1}:s_{k}} \equiv A_{s_{1}},A_{s_{1}+1},A_{s_{1}+2},\mydots,A_{s_{k}}$. By itself, $1:i$ has no meaning, and it does not mean $1,\mydots,i$. 
For $i\in\mathbb{N}$, we let $[i] \triangleq \{1,\mydots,i\}$.\\ 

 \textbf{Generalized inner-product.} Given two equi-multidimensional arrays $A$ and $B$, and $\ell\in\mathbb{N}$, we define \[\inner{A}{B}_{\ell} 
 \triangleq \sum_{s_{1:k}} (A_{s_{1:k}})^\ell B_{s_{1:k}} 
 =  \sum^{\XU{1}}_{s_1=1}\dots  \sum^{\XU{k}}_{s_k=1}(A_{s_1,\dots,s_k})^\ell B_{s_1,\dots,s_k},\] where $(\cdot)^\ell$ is element-wise  $\ell$th power. Later on, $A$ will be an array of distances and $B$ will be an array of probabilities.
 
 \textbf{Probability spaces.} To facilitate exposition, we state our main contributions for probability spaces with a  finite sample space in $\X$, an event set $\sigma$-algebra which is the power set of the sample space, and a probability measure described by a probability mass function\ignore{, but they extend to more general settings}. We refer to probability mass functions using bold letters, e.g. $\PI$, $\QI$, $\RI$, etc.

When talking about $\NumDist$ probability spaces, the $i$th space has sample space $\XU{i} = \{\XUL{i}{1:\SizeDistU{i}}\} \subseteq \X $, an event space $2^{\XU{i}}$, 
and a probability mass function $\PIU{i}$, or $\QIU{i}$, or $\RIU{i}$, etc. Variable $\SizeDistU{i}$ is the number of atoms in $\XU{i}$.
Symbol $\PIUL{i}{s}$ denotes the probability of the atomic event $\{\XUL{i}{s}\}$. Without loss of generality (w.l.o.g.) we assume $\PIUL{i}{s} > 0, \forall i\in[\NumDist], \forall s\in[\SizeDistU{i}]$.
Our notation assumes that atoms can be indexed, but our results extend beyond this assumption.
W.l.o.g., we assume that $\XUL{i}{s} = \XUL{i}{t}$ if and only if $s = t$.

Symbol $\PIU{i_{1:k}}$ denotes a mass function for the probability space with sample space $\XU{i_1}\times\mydots\times\XU{i_k}$ and event space $2^{\XU{i_1}\times\mydots\times\XU{i_k}}$. In particular, $\PIUL{i_{1:k}}{s_{1:k}}$ (i.e. $\PIUL{i_{1},\mydots,i_k}{s_1,\mydots,s_{k}}$) is the probability of the atomic event $\{(\XUL{i_1}{s_1},\mydots,\XUL{i_k}{s_k})\}$.
We use $\PIU{i_{1:k} \mid j_{1:r}}$ to denote a probability mass function for the probability space with sample space $\XU{i_1}\times\mydots\times\XU{i_k}$, and event space $2^{\XU{i_1}\times\mydots\times\XU{i_k}}$, such that $\PIUL{i_{1:k} \mid j_{1:r}}{s_{1:k} \mid t_{1:r}} \triangleq {\PIUL{i_1,\mydots,i_k , j_1,\mydots,j_r}{s_1,\mydots,s_k , t_1,\mydots,t_r} }/{\PIUL{j_1,\mydots,j_r}{ t_1,\mydots,t_r} }$, i.e. a conditional probability. 

\begin{definition}[Gluing map]\label{def:new_gluing_map} Consider a mass function $\QIU{k}$ over  $\XU{k}$ and $\NumDist-1$ conditional mass functions $\{\QIU{i \mid k}\}_{i \in [\NumDist]\backslash\{k\}}$ over $\{\XU{i}\}_{i \in [\NumDist]\backslash\{k\}}$. The map $\G$,  that is called a gluing map, defines the  mass function $\PI$ over $\XU{1}\times\mydots \times\XU{\NumDist}$ as
\vspace{-0.15cm}
\begin{align} \label{eq:def_of_G_1}
 \PI&=\Gmap{\QIU{k},\{\QIU{i \mid k}\}_{i \in [\NumDist]\backslash \{k\}}} 
 . 
 \end{align}
%
To be more specific,
$
\PIUL{}{s_1,\mydots,s_n} = \QIUL{k}{s_k} \prod_{i \in [\NumDist]\backslash \{k\}} \QIUL{i\mid k}{s_i }.
$
\end{definition}

\textbf{Distances and metrics.} We use ``distance'' to refer to an object that, depending on extra assumptions, might, or might not,  have the properties of a (generalized) metric. For a metric we use the standard definition, and for generalized metrics we use the definitions in \cite{kiss2018generalization}.

The definition in \cite{kiss2018generalization} is the same as ours but is expressed in a slightly different form. In particular, the $C(n)$ in our Definition \ref{def:gen_metric_d} will correspond to $1/K_n$ in \cite{kiss2018generalization}. Also, \cite{kiss2018generalization} expresses their generalized triangle inequality by replacing the $r$th variable inside a multi-distance with a new variable, while in our definition we remove the $r$th variable and append a new variable at the end of the argument's list. Given that both definitions also require permutation invariance of arguments to the multi-distance, these two definitions are equivalent.

Finally, our notions of metric and generalized metric are more general than usual in the sense that they support the use of different distance functions depending on the spaces from  where we are drawing elements. This grants an extra layer of generality to our results.

\begin{definition}[Metric]\label{def:classic_metric}
Let $\ignore{\Dist  = }\{\DistU{i,j}\}_{i,j}$ be a set of distances of the form $\DistU{i,j}:\XU{i}\times\XU{j}\mapsto \mathbb{R}$ and $\DistU{i,j}(\XUL{i}{s},\XUL{j}{t}) \triangleq \DistUL{i,j}{s,t}$.
We say that \emph{$\DistU{}$ is a metric} when, for any $i,j,k$, and any $s\in [\SizeDistU{{i}}], r\in [\SizeDistU{{j}}], t\in [\SizeDistU{{k}}]$, we have:  1) $\DistU{i,j}_{r,s} \geq 0$; 2) $\DistU{i,j}_{r,s} =  \DistU{j,i}_{s,r}$; 3) $\DistU{i,j}_{s,r} = 0 \text{ iff } \XUL{i}{s} = \XUL{j}{r}$; 4) $\DistU{i,j}_{s,r} \leq \DistU{i,k}_{s,t} + \DistU{k,j}_{t,r}$.
\end{definition}

\begin{definition}[Generalized metric]\label{def:gen_metric_d}
Let $\ignore{\Dist  = }\{\DistU{i_{1:\NumDist}}\}_{i_{1:\NumDist}}$ be a set of distances of the form $\DistU{i_{1:\NumDist}}:\XU{i_1}\times\mydots\times\XU{i_\NumDist} \mapsto \mathbb{R}$ and $\DistU{i_{1:\NumDist}}(\XUL{i_1}{s_1},\mydots,\XUL{i_k}{s_\NumDist}) \triangleq \DistUL{i_{1:\NumDist}}{s_{1:\NumDist}}$. We say that $\Dist$ is an $(\NumDist,C(\NumDist))$-metric when, for any $i_{1:\NumDist+1}$ and $s_{1:\NumDist+1}$ with $s_r \in [\SizeDistU{i_r}]\forall r\in[\NumDist+1]$,  we have:\; 1) $\DistUL{i_{1:\NumDist}}{s_{1:\NumDist}} \geq 0$;\; 2) $\DistUL{i_{1:\NumDist}}{s_{1:\NumDist}} = \DistUL{\sigma(i_{1:\NumDist})}{\sigma(s_{1:\NumDist})}$ for any permutation  $\sigma$;\; 3) $\DistUL{i_{1:\NumDist}}{s_{1:\NumDist}} = 0 \text{ iff } \XUL{i_r}{s_r} = \XUL{i_t}{s_t},\; \forall  r,t\in[\NumDist]$;\; 4)  $C(\NumDist)\DistUL{i_{1:\NumDist}}{s_{1:\NumDist}} \leq \sum^{\NumDist}_{r=1} \DistUL{i_1,\mydots,i_{r-1},i_{r+1},\mydots,i_{\NumDist+1}}{s_1,\mydots,s_{r-1},s_{r+1},\mydots,s_{\NumDist+1}}$.
\end{definition}

\begin{definition}[Generalized metric on distributions]\label{def:gen_metric_prop_W}
Let $\Wdist{}$ be a map from $\NumDist$ probability spaces to $\mathbb{R}$ such that
$\Wdist{i_{1:\NumDist}} \triangleq \W(\PIU{i_{1:\NumDist}}) $ is the image of the probability spaces with indices $i_{1:\NumDist}$. For any $\NumDist+1$ probability spaces with samples  $\XU{1:\NumDist+1}$ and masses $\PIU{1:\NumDist+1}$, and any permutation $\sigma$,$\Wdist{}$ is an \emph{$(\NumDist,C(\NumDist))$-metric} if: \;
1) $\WU{1,\mydots,\NumDist} \geq 0$ \label{th:n_metric_i}; \;
2)
 $\WU{1,\mydots,\NumDist} = 0 \text{ iff } \PIU{i} = \PIU{j}$, $\XU{i}=\XU{j}, \ \forall i,j$\label{th:n_metric_iii};\;
3) $\WU{1,\mydots,\NumDist} = \WU{\sigma(1,\mydots,\NumDist)}$ \label{th:n_metric_ii};\;
4) $C(\NumDist) \WU{1,..,\NumDist} \leq \hspace{-0.14cm}\sum\limits^{\NumDist}_{r = 1} \WU{1,..,r-1,r+1,..,\NumDist+1}$. \label{th:n_metric_iv}

\end{definition}

\begin{remark}
Equalities $\PIU{i} = \PIU{j}$ and $\XU{i}=\XU{j}$, mean that
$\SizeDistU{i} = \SizeDistU{j}$, and that there exists a bijection $b^{i,j}(\cdot)$ from $[\SizeDistU{i}]$ to $[\SizeDistU{j}]$ such that 
$\PIUL{i}{s} = \PIUL{j}{b^{i,j}(s)}$ and $\XUL{i}{s}=\XUL{j}{b^{i,j}(s)}$, $\forall$ $s \in [\SizeDistU{i}]$.
\end{remark}
\begin{remark}
The inequality in the property 4 in Definition 4 is the generalized triangle inequality. Figure \ref{app:fig:geometric_analogy} depicts a geometric analogy for the generalized triangle inequality. Just like the measure (length) of one of the sides of
a triangle is always smaller than the sum of the measure of the two other sides, the measure (area/volume/etc) of one of the facets of a simplex is always smaller than the sum of the measure of the other
facets. The constant $C(n)$ is included in our definition so that we later can prove that a stronger inequality than $C(n) =1$ holds.
\end{remark}

\begin{figure}
\centering
{\includegraphics[scale=0.27]{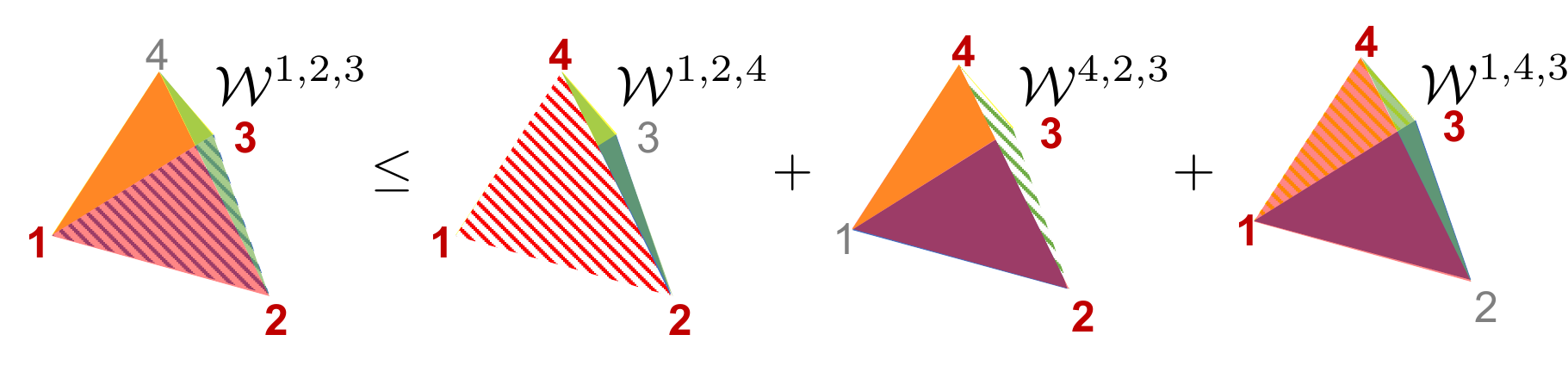}}
\caption{\small  Geometric analog of the generalized triangle inequality: the total area of any three faces in a tetrahedron is greater than that of the fourth face. 
} \label{app:fig:geometric_analogy}
\end{figure}
We abbreviate $(\NumDist,1)$-metric by $\NumDist$-metric. In our setup, the $\inf$ in \eqref{eq:multi_ot} is  always attained (recall, finite-spaces) and amounts to solving an LP. We refer to the minimizing  distributions by $\PIU{\optsymbol}, \QIU{\optsymbol}, \RIU{\optsymbol}$, etc. We define the following map from $\NumDist$ probability spaces to $\mathbb{R}$. The definition below amounts to \eqref{eq:multi_ot} when $\PI$'s are \emph{empirical measures}.
\begin{definition}[MMOT distance for finite spaces]\label{def:discrete_MMOT}
 Let $\ignore{\Dist = }\{\DistU{i_{1:\NumDist}}\}_{i_{1:\NumDist}}$ be a set of distances of the form $\DistU{i_{1:\NumDist}}:\XU{i_1}\times\mydots\times\XU{i_\NumDist}\mapsto\mathbb{R}$ and $\DistU{i_{1:\NumDist}}(\XUL{i_1}{s_1},\mydots,\XUL{i_k}{s_\NumDist}) \triangleq \DistUL{i_{1:\NumDist}}{s_{1:\NumDist}}$. 
The \emph{MMOT distance} associated with $\ignore{\Dist}\{\DistU{i_{1:\NumDist}}\}_{i_{1:\NumDist}}$  for $\NumDist$ probability spaces with masses $\PIU{i_{1:\NumDist}}$ over $\XU{i_{1:\NumDist}}$  is
\begin{equation} \label{eq:mmot_discrete}
\W(\PIU{i_{1:\NumDist}}) \triangleq \WU{i_{1:\NumDist}} = \min_{\RIU{}:\RIU{i_s} = \PIU{i_s} \forall s \in [\NumDist]} \inner{\DistU{i_{1:\NumDist}}}{\RI}_{\ell}^{\frac{1}{\ell}},
\end{equation}
where $\RIU{}$ is a mass function over $\XU{i_1}\times\mydots\times\XU{i_\NumDist}$, and  $\RIU{i_s}$ be the marginal probability of $\RI$ on $\XU{i_s}$.
\end{definition}
\begin{remark}
Solving \eqref{eq:mmot_discrete} amounts to solving a linear program when $\ell = 1$, and to solving a convex optimization problem when $\ell>1$ and $d$ is non-negative.
Solving a linear program can be done in polynomial time in the number of variables \cite{Karmarkar84anew}. However, for general MMOT, the number of
variables is exponential in the number of marginals $n$ and their support sizes $k$, which makes the overall run time be $\text{poly}(k^n)$. Even approximately solving MMOT in some general instances is NP-hard. However, under certain structural assumptions on the input costs/distances, we
can get tractable algorithms \cite{altschuler2021hardness,lin2019complexity}.
\end{remark}

\section{Main results}\label{sec:result}
To prove that MMOT leads to an $\NumDist$-metric, it is natural to extend the ideas in the classical proof that  WD is a metric. The hardest property to prove is the triangle inequality. 
Its proof for the WD follows from  a gluing lemma and the assumption that  $\Dist$ itself is a metric (Def. \ref{def:classic_metric}). Our hope is that we can prove the triangle inequality for MMOT if we can prove (a) a generalized gluing lemma and assume that (b) $\Dist$ is a generalized metric. Unfortunately, and as we explain in Sections \ref{sec:no_gluing} and \ref{sec:d_n_metric_not_enough}, (a) is not possible, and (b) is not sufficient. This requires developing completely new proofs.
%
\subsection{The gluing lemma does not generalize to higher dimensions}\label{sec:no_gluing}
The gluing lemma used to prove that WD is a metric is as follows. For its proof see  \cite{ambrosio2013user}, Lemma 2.1.
\begin{lemma}[Gluing lemma]\label{app:th:gluing_lemma}
Let $\PIU{1,3}$ and $\PIU{2,3}$ be arbitrary mass functions for $\XU{1}\times \XU{3}$ and $\XU{2}\times \XU{3}$, respectively, with the same  marginal, $\PIU{3}$, over $\XU{3}$. There exists a mass function 
 $\RIU{1,2,3}$ for $\XU{1}\times \XU{2}\times\XU{3}$ whose marginals over $\XU{1}\times\XU{3}$ and $\XU{2}\times\XU{3}$ equal $\PIU{1,3}$ and $\PIU{2,3}$ respectively.
\end{lemma}

The way Lemma \ref{app:th:gluing_lemma} is used to prove WD's triangle inequality is as follows. Assume $\Dist$ is a metric (Def. \ref{def:classic_metric}). Let $\ell = 1$ for simplicity. Let $\PIsU{1,2}$, $\PIsU{1,3}$, and $\PIsU{2,3}$ be optimal transports such that $\WU{1,2} = \inner{\PIsU{1,2}}{\DistU{1,2}}$, $\WU{1,3} =\inner{\PIsU{1,3}}{\DistU{1,3}}$, and $\WU{2,3} =\inner{\PIsU{2,3}}{\DistU{2,3}}$.
Define $\RIU{1,2,3}$ as in Lemma \ref{app:th:gluing_lemma}, and let $\RIU{1,3}$, $\RIU{2,3}$, and $\RIU{1,2}$ be its bivariate marginals. We then have 
\begin{align*}
&\inner{\PIsU{1,2}}{\DistU{1,2}} \stackrel{\text{sub-optimal } \RI}{\leq} \inner{\RIU{1,2}}{\DistU{1,2}} \hspace{-0.cm}=\hspace{-0.cm}\sum_{s,t} \RIUL{1,2}{s,t}\DistUL{1,2}{s,t} \hspace{-0.cm}=\hspace{-0.cm} \sum_{s,t,l} \RIUL{1,2,3}{s,t,l}\DistUL{1,2}{s,t}   
\\
&\stackrel{\text{$d$ is metric}}{\leq}\hspace{-0.1cm}
\sum_{s,t,l} \RIUL{1,2,3}{s,t,l}(\DistUL{1,3}{s,l} + \DistUL{2,3}{t,l})
=
\inner{\RIU{1,3}}{\DistU{1,3}} + 
\inner{\RIU{2,3}}{\DistU{2,3}} \\
&\stackrel{\text{Lemma \ref{app:th:gluing_lemma}}}{=}\inner{\PIsU{1,3}}{\DistU{1,3}} + \inner{\PIsU{2,3}}{\DistU{2,3}}.
\end{align*}

Our first roadblock is that Lemma \ref{app:th:gluing_lemma} does not generalize to higher dimensions. For simplicity, we now omit the sample spaces on which mass functions are defined. When a set of mass functions have all their marginals over the same sample sub-spaces equal, we will say they are \emph{compatible}.

\begin{theorem}[No gluing]\label{thm:no_gluing}
There exists mass functions $\PIU{1,2,4}$, $\PIU{1,3,4}$, and $\PIU{2,3,4}$ with compatible marginals such that there is no mass function $\RIU{1,2,3,4}$ compatible with them.
\end{theorem}
%
%
%
\vspace{-10pt}
\begin{proof}
If this were not the case, then, by computing marginals of the mass functions in the theorem's statement, it would be true that, given arbitrary mass functions $\PIU{1,2}$, $\PIU{1,3}$, and $\PIU{2,3}$ with compatible univariate marginals, we should be able to find $\RIU{1,2,3}$ whose bivariate marginals equal these three mass functions. But this is not the case. A counter example,  taken from Remark 3 in \cite{haasler2021multimarginal}, is given below.
%
For example, let
\begin{align*}
 \PIU{1,2} = \PIU{1,3} =
\frac{1}{2}\begin{bmatrix}
1 & 0\\
0 & 1
\end{bmatrix}\quad \text{and}\quad
\PIU{2,3}=
\frac{1}{2}\begin{bmatrix}
0 & 1\\
1 & 0
\end{bmatrix}.
\end{align*}
%
These marginals have compatible univariate marginals, namely, $\PIU{1} = \PIU{2} = \PIU{3} = [1, 1]/2$. Yet, 
 the following system of eqs. over $\{\RIUL{1,2,3}{i,j,k}\}_{i,j,k \in [2]}$ is easily checked to be infeasible
 \begin{align*}
     &\left(\RIUL{1,2,3}{i,j,k} \geq 0 \; \forall i,j,k\right) \land \Big(\sum_{i}\RIUL{1,2,3}{i,j,k} = \PIUL{2,3}{j,k} \; \forall j,k \Big) \land \Big( \sum_{j}\RIUL{1,2,3}{i,j,k} = \PIUL{1,3}{i,k}  \;\forall i,k \Big)\\
     &  \land \Big( \sum_{k}\RIUL{1,2,3}{i,j,k} = \PIUL{1,2}{i,j}  \;\forall i,j \Big).
 \end{align*}
\end{proof}
\begin{remark}
It might be possible to obtain a generalized gluing lemma if, in addition to requiring consistency of the univariate marginal distributions, we also require consistency of the $(n-1)$-variate marginal distributions. We leave  studying properties of MMOTs defined with additional $k$-variate marginal consistency constraints as future work.
\end{remark}
\subsection{Cost $d$ being an $\NumDist$-metric is not a sufficient condition for MMOT to be an $\NumDist$-metric} \label{sec:d_n_metric_not_enough}
%
Theorem \ref{thm:no_gluing} tells us that, even if we assume that $\Dist$ is an $\NumDist$-metric, we cannot adapt the classical proof showing WD is a metric to MMOT leading to an $\NumDist$-metric. The question remains, however, whether there exists such a proof at all only under the assumption that $\Dist$ is $\NumDist$-metric. Theorem \ref{app:th:counter_example} settles this question in the negative.
\begin{theorem} \label{app:th:counter_example}
Let $\W$ be  MMOT distance as in Def. \ref{def:discrete_MMOT} with $\ell =1$. There exists a sample space $\XU{}$, mass functions $\PIU{1}$, $\PIU{2}$, $\PIU{3}$, and $\PIU{4}$ over $\XU{}$, and distance $\Dist:\XU{}\times\XU{}\times\XU{} \mapsto \mathbb{R}$ such that $\Dist$ is an $\NumDist$-metric ($\NumDist = 3$), but 
\begin{align}\label{eq:triang_ineq_viol}
\WU{1,2,3} > \WU{1,2,4} + \WU{1,3,4} + \WU{2,3,4}.  
\end{align}
\end{theorem}

\begin{proof}
Let $\X$ be the six points in Figure \ref{app:fig:counter_example}, where we assume that $0 <  \epsilon \ll 1$, and hence that there are no three co-linear points, and no two equal points. Let $\PIU{1},\PIU{2},\PIU{3}$, and $\PIU{4}$ be as in Figure \ref{app:fig:counter_example}, each is represented by a unique color and is uniformly distributed over the points of the same color.
Given any $x,y,z \in \X$ let $\Dist(x,y,z) = \gamma$ if exactly two points are equal, and let $\Dist(x,y,z)$ be the area of the corresponding triangle otherwise, where $\gamma$ lower bounds the area of the triangle formed by any three non-co-linear points, e.g. $\gamma = \epsilon / 4$. 
A few geometric considerations (see Appendix \ref{app:proof_details_of_counter_example}) show that $\Dist$ is an $\NumDist$-metric ($\NumDist=3, C(n) = 1$) and that \eqref{eq:triang_ineq_viol} holds as $\frac{1}{2} > \frac{1}{8} + \left(\frac{1}{8} + \frac{\epsilon}{4} \right)+ \left(\frac{1}{8} + \frac{\epsilon}{4}\right)$. 
\end{proof}

\begin{figure}
%
\centering
{\includegraphics[scale=0.32,trim={0.0cm 0cm 0cm 0.cm},clip]{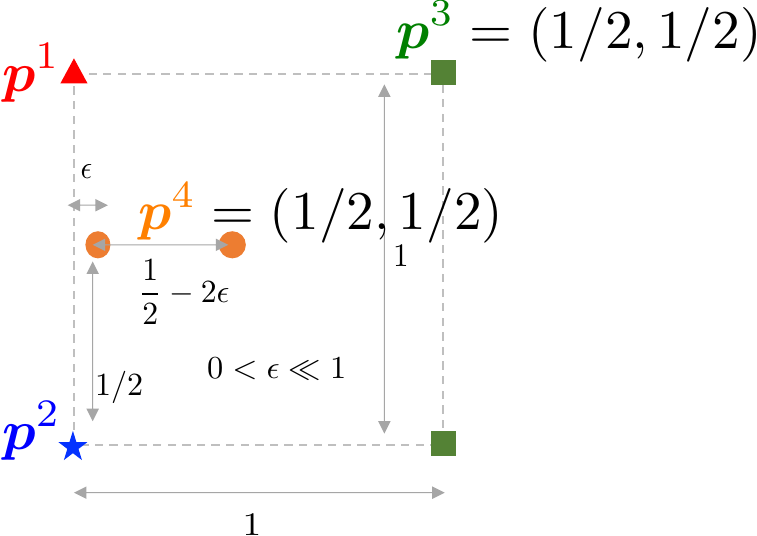}}\quad\quad
{\includegraphics[scale=0.32,trim={0.0cm -3.5cm 0cm 9cm},clip]{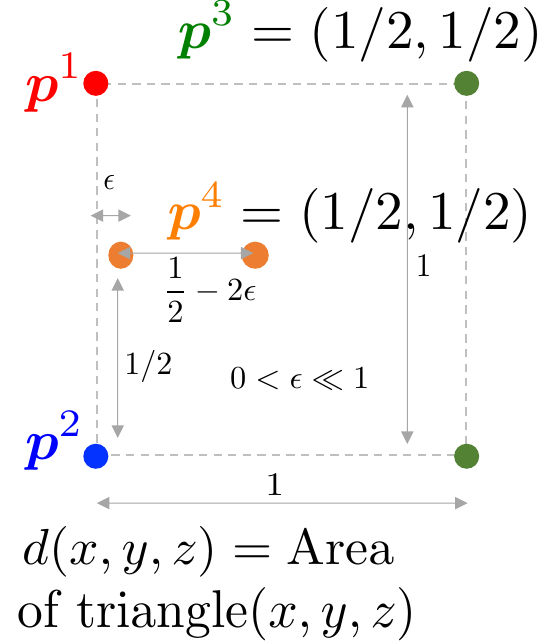}}
%
%
\caption{\small  Sample space $\X$, mass functions $\{\PIU{i}\}^4_{i=1}$, and cost function $\Dist$ that lead to violation \eqref{eq:triang_ineq_viol}. 
Red triangle is ${\bm p}^1$. Blue star is ${\bm p}^2$.  Green squares are ${\bm p}^3$. Yellow circles are ${\bm p}^4$.}
\label{app:fig:counter_example}
\end{figure}
%
\begin{remark}
Theorem \ref{app:th:counter_example} can be generalized to spaces of dimension $> 2$, and to $\NumDist > 3$, and $\ell > 1$.
\end{remark}
%
\subsection{A new family of pairwise MMOTs that are generalized metrics}\label{sec:MMOT_metric_properties_main}
We will prove that the properties of a generalized metric (Def. \ref{def:gen_metric_prop_W}) hold for the novel family of MMOT distances defined as follows.

\begin{definition}[Pairwise MMOT distance]\label{def:pair_wise_MMOT}
 Let $\{\DistU{i,j}\}_{i,j}$ be a set of distances of the form $\DistU{i,j}:\XU{i}\times\XU{j}\mapsto\mathbb{R}$ and $\DistU{i,j}(\XUL{i}{s},\XUL{j}{t}) \triangleq \DistUL{i,j}{s,t}$. 
The \emph{Pairwise MMOT distance} associated with $\ignore{\Dist}\{\DistU{i,j}\}_{i,j}$ for $\NumDist$ probability spaces with masses $\PIU{i_{1:\NumDist}}$ over $\XU{i_{1:\NumDist}}$  is $\W(\PIU{i_{1:\NumDist}})  \triangleq \WU{i_{1:\NumDist}}$ with
%
%
\begin{align} \label{eq:main_def_of_W}
\WU{i_{1:\NumDist}} 
= \hspace{-0.3cm}\min_{\substack{\RIU{}:\RIU{i_s}  = \PIU{i_s}\\ \forall s \in [\NumDist]}} \sum_{ 1 \leq s < t \leq \NumDist}\hspace{-0.3cm} \inner{\DistU{i_s,i_t}}{\RIU{i_s,i_t}}_{\ell}^{\frac{1}{\ell}}, 
\end{align}
where $\RI$ is a mass over $\XU{i_1}\times\mydots\times\XU{i_\NumDist}$, with
marginals
$\RIU{i_s}$ and $\RIU{i_s,i_t}$ over $\XU{i_s}$ and  $\XU{i_s}\times\XU{i_t}$.
\end{definition}
\begin{remark}
Note that each choice of $\{\DistU{i,j}\}_{i,j}$ results in a different MMOT. Our results in this section hold for the whole family of such MMOTs.  \end{remark}
\begin{remark} \label{rmk:comment_on_sum_of_inde_pairwise_sum}
Swapping $\min$ and $\sum$ gives a new definition $\WU{i_{1:\NumDist}}_{\text{pairs}} = \sum_{1 \leq s<t\leq\NumDist} \WU{i_s,i_t}$, where $\WU{i_s,i_t}$ is the WD between \ignore{the $i_s$th and $i_t$th spaces}$\XU{i_s}$ and $\XU{i_t}$. This is trivially an $\NumDist$-metric (cf. \cite{kiss2018generalization}) but is different from eq. \eqref{eq:main_def_of_W}. In particular, it does not provide a joint optimal transport, which is important to many applications such as color transfer \cite{strossner2022low}, tomographic reconstruction \cite{abraham2017tomographic}, and robust localization and sensor fusion \cite{elvander2020multi}.
\end{remark}

If $\NumDist = 2$, Def. \ref{def:pair_wise_MMOT} reduces to the Wasserstein distance.
%
Our definition is a special case of the Kantorovich formulation for the general MMOT problem discussed in \cite{pass2015multi}.
Furthermore, if $\ell = 1$, we can get pairwise MMOT (Def. \ref{def:pair_wise_MMOT}) from general MMOT (Def. \ref{def:discrete_MMOT}), by  defining $\DistU{i_{1:\NumDist}}:\XU{i_1}\times\mydots\times\XU{i_\NumDist} \mapsto \mathbb{R}$ such that
\begin{equation} \label{eq:choice_to_relate_both_MMOTs}
 \DistU{i_{1:\NumDist}}(w^{1:\NumDist}) = \left(\sum_{1\leq s<t\leq\NumDist}(\DistU{i_s,i_t}(w^s,w^t))^{\ell}\right)^{\frac{1}{\ell}}, 
\end{equation}
for some set of distances $\{\DistU{i,j}\}_{i,j}$.
However, if $\ell > 1$, choosing $\DistU{i_{1:\NumDist}}$ as in eq. \eqref{eq:choice_to_relate_both_MMOTs} in Def. \ref{def:discrete_MMOT} does not lead to Def. \ref{def:pair_wise_MMOT}, but rather to an upper bound of it.
Indeed, with this choice we get 
\begin{align}
 &\min_{{ \boldsymbol r}:{ \boldsymbol r}^{i_s} = { \boldsymbol p}^{i_s},\; \forall s \in [n]} \langle  d^{i_{1:n}} , { \boldsymbol r}\rangle^{1/\ell}_{\ell} = \\
 &\min_{{ \boldsymbol r}:{ \boldsymbol r}^{i_s}={ \boldsymbol p}^{i_s},\; \forall s \in [n]} \left( \sum_{w^{1:n}}  \sum_{1 \leq s < t \leq n } (d^{i_s,i_t}(w^s , w^t))^{\ell} { \boldsymbol r}(w^{1:n})\right)^{1/\ell}=\\
 &\min_{{ \boldsymbol r}:{ \boldsymbol r}^{i_s}={ \boldsymbol p}^{i_s},\; \forall s \in [n]} \left(  \sum_{1 \leq s < t \leq n } \sum_{w^{1:n}}  (d^{i_s,i_t}(w^s , w^t))^{\ell} { \boldsymbol r}(w^{1:n})\right)^{1/\ell}=\\
 &\min_{{ \boldsymbol r}:{ \boldsymbol r}^{i_s}={ \boldsymbol p}^{i_s} \forall s \in [n]} \left(  \sum_{1 \leq s < t \leq n } \sum_{w^{s},w^{t}}  (d^{i_s,i_t}(w^s , w^t))^{\ell} { \boldsymbol r}^{i_s,i_t}(w^{s},w^{t})\right)^{1/  \ell}=\\
 &\min_{{ \boldsymbol r}:{ \boldsymbol r}^{i_s}={ \boldsymbol p}^{i_s} \forall s \in [n]} \left(  \sum_{1 \leq s < t \leq n } \langle d^{i_s,i_t} , { \boldsymbol r}^{i_s,i_t}\rangle_{\ell} \right)^{1/\ell}\\
 & \leq \min_{{ \boldsymbol r}:{ \boldsymbol r}^{i_s}={ \boldsymbol p}^{i_s} \forall s \in [n]}  \sum_{1 \leq s < t \leq n } \langle d^{i_s,i_t} , { \boldsymbol r}^{i_s,i_t}\rangle^{1/\ell}_{\ell},
\end{align}
where the last inequality follows from Hölder's inequality.

It is easy to prove that if $\{\DistU{i,j}\}_{i,j}$ is a metric (Def. \ref{def:classic_metric}), then $\Dist$ is an $\NumDist$-metric (Def. \ref{def:gen_metric_d}). However, because of Theorem \ref{app:th:counter_example}, we know that this is not sufficient to guarantee that the pairwise MMOT distance is an $\NumDist$-metric, which only makes the proof of the next theorem all the more interesting.

\begin{theorem}\label{th:n_metric}
If $\Dist$ is a metric (Def. \ref{def:classic_metric}), then the pairwise MMOT distance (Def. \ref{def:pair_wise_MMOT}) associated with $\Dist$ is an $(\NumDist,C(\NumDist))$-metric, with $C(n) \geq 1$. 
\end{theorem}
%
\begin{proof}
The proof is presented in Appendix \ref{app:proof_of_n_metric}.
\end{proof}
%
We currently do not know the most general conditions under which Def. \ref{def:gen_metric_d} is an $\NumDist$-metric. However, working with Def. \ref{def:pair_wise_MMOT} allows us sharply bound the best possible $C(n)$, which would unlikely be possible in a general setting. As Theorem \ref{th:n_metric_c} shows, the best $C(n)$ is $C(n) = \Theta(n)$.
\begin{theorem}\label{th:n_metric_c}
In Theorem \ref{th:n_metric}, the constant $C(\NumDist)$ can be made larger than $(\NumDist-1)/5$ for $\NumDist > 7$, and there exists sample spaces $\XU{1:\NumDist}$, mass functions $\PIU{1:\NumDist}$, and a metric $d$ over $\XU{1:\NumDist}$ such that  $C(\NumDist) \leq n-1$.
\end{theorem}
%
\begin{proof}
The proof is presented in Appendix \ref{app:proof_of_n_metric_C_upper_bound}.
\end{proof}
%
\begin{remark}\label{rmk:barycenter}
Note that if $\XU{i} = \X,\ \forall\ i$ and $\Dist: \X\times\dots\times\X \mapsto \mathbb{R}$ such that $\Dist(w^{1:\NumDist}) = \min_{w \in \X}\sum_{s \in [\NumDist]} \DistU{1,2}(w^s,w)$ and $\DistU{1,2}$ is a metric, then $\Dist$ is an $\NumDist$-metric \cite{kiss2018generalization}. One can then prove, see \cite{carlier2010matching}, that Def. \ref{eq:mmot_discrete} is equivalent to $\Wdist{}(\PIU{{1:\NumDist}}) =\min_{\PI} \sum_{s\in [\NumDist]}\Wdist{}(\PIU{s},\PI)$, which is also called the \emph{Wasserstein barycenter distance} (WBD)
\cite{agueh2011barycenters}. 
The later definition makes $\Wdist{}(\PIU{{1:\NumDist}})$
a Fermat $n$-distance as defined in \cite[Proposition 3.1]{kiss2018generalization}, from which it follows immediately via this same proposition that it is an
$n$-metric with $C(n) = \Theta(n)$. The pairwise MMOT is not a Fermat distance, and Theorems \ref{th:n_metric} and \ref{th:n_metric_c} do not follow from \cite{kiss2018generalization}. Hence, a novel proof strategy is required.
\end{remark}

In the next section, we give a self contained
 proof that the generalized triangle inequality holds with $C(\NumDist) = 1$ for $\NumDist = 3$ 
when $\Dist$ is a metric. This proof contains the
key ideas behind the proof of the triangle inequality in Theorems \ref{th:n_metric} and \ref{th:n_metric_c}. Proving the generalized triangle inequality for a general $\NumDist$, and a large $C(\NumDist)$, is the hardest part of these theorems, compared to proving the other properties required by generalized metrics.
%
\section{Proof of the generalized triangle inequality for $\NumDist = 3$, $\ell=1$, and $C(\NumDist) = 1$}\label{sec:proof_n_3_ell_1}
 %
We will prove that for any mass functions $\PIU{1},\dots,\PIU{4}$ over
$\XU{1},\dots,\XU{4}$, respectively,
if $\DistU{i,j}:\XU{i}\times\XU{j}\mapsto \mathbb{R}$ is a metric for
any $i,j\in\{1,\dots,4\}, i \neq j$, then $\WU{1,2,3} \leq \WU{1,2,4} + \WU{1,3,4} + \WU{2,3,4}$.

We write this inequality more succinctly
as
\begin{equation}\label{eq:simple_proof_first_step}
\WU{1,2,3} \leq \WU{\backslash 3} + \WU{\backslash 2} + \WU{\backslash 1},
\end{equation}
 using a symbol $\WU{\backslash r}$ whose meaning is obvious.
We begin by expanding all  the terms in  \eqref{eq:simple_proof_first_step},
\begin{align*}
    &\WU{1,2,3} = \inner{\DistU{1,2}}{\PIsU{1,2}}+    \inner{\DistU{1,3}}{\PIsU{1,3}}+\inner{\DistU{2,3}}{\PIsU{2,3}},
    \end{align*}
    and,
    \begin{align*}
      &\WU{\backslash 3} + \WU{\backslash 2} + \WU{\backslash 1} \\&
      =\inner{\DistU{1,2}}{\PIsU{(3)^{1,2}}}+    \inner{\DistU{1,4}}{\PIsU{(3)^{1,4}}}+\inner{\DistU{2,4}}{\PIsU{(3)^{2,4}}}\\
    &+\inner{\DistU{1,3}}{\PIsU{(2)^{1,3}}}+    \inner{\DistU{1,4}}{\PIsU{(2)^{1,4}}}+\inner{\DistU{3,4}}{\PIsU{(2)^{3,4}}}\\
    &+ \inner{\DistU{2,3}}{\PIsU{(1)^{2,3}}}+    \inner{\DistU{2,4}}{\PIsU{(1)^{2,4}}}+\inner{\DistU{3,4}}{\PIsU{(1)^{3,4}}},
\end{align*}

where $\{\PIsU{i,j}\}$ are the bivariate
marginals of the optimal
joint distribution $\PIsU{1,2,3}$ for $\WU{1,2,3}$,
and $\{\PIsU{(r)^{i,j}}\}$ are the
bivariate marginals of the optimal
joint distribution for $\WU{\backslash r}$.

Now we define the following
probability mass function
on $\XU{1}\times\dots\times\XU{4}$, namely $\PIU{1,2,3,4}$,  such that
\begin{align}\label{eq:joint_probability_mass_function}
\PIUL{1,2,3,4}{s,t,l,u}
= \PIUL{4}{u} \frac{\PIsUL{(3)^{1,4}}{s,u}}{\PIUL{4}{u}}\frac{\PIsUL{(2)^{3,4}}{l,u}}{\PIUL{4}{u}}\frac{\PIsUL{(1)^{2,4}}{t,u}}{\PIUL{4}{u}}.    
\end{align}
Notice that this definition is such that its bivariate marginals match the optimal bivariate marginals of each of the terms in the generalized triangle inequality.
Recall that w.l.o.g. we assume that no element in $\XU{i}$ has zero mass, so the denominators are not zero.
Since the bivariate probability mass functions  
$\PIU{1,2},\PIU{1,3}$, and $\PIU{2,3}$ of
$\PIU{1,2,3,4}$ are  feasible but sub-optimal choices
of minimizers in eq. \eqref{eq:main_def_of_W} in Def. \ref{def:pair_wise_MMOT}, we have that
\begin{align} \label{eq:simple_proof_mid_step}
\WU{1,2,3}\hspace{-0.cm}\leq\hspace{-0.cm}
\inner{\DistU{1,2}}{\PIU{1,2}}\hspace{-0.05cm}+\hspace{-0.05cm}    \inner{\DistU{1,3}}{\PIU{1,3}}\hspace{-0.05cm}+\hspace{-0.05cm}\inner{\DistU{2,3}}{\PIU{2,3}}.
\end{align}
It is convenient to introduce the following more compact notation 
$w_{i,j} = \inner{\DistU{i,j}}{\PIU{i,j}}$ and
$w^*_{i,j,(r)} = \inner{\DistU{i,j}}{\PIsU{(r)^{i,j}}}$.
Notice that, for any $i,j,k$ and $r$,  we have
 $w_{i,j}\leq w_{i,k} + w_{j,k}$ and  $w^*_{i,j,(r)}\leq w^*_{i,k,(r)} + w^*_{j,k,(r)}$. 
This follows directly from the assumption
that $\{\DistU{i,j}\}$ are metrics.
Let us prove that $w_{1,2}\leq w_{1,4} + w_{2,4}$:
\begin{align}
&w_{1,2} = \sum_{s,t} \DistUL{1,2}{s,t}\PIUL{1,2}{s,t}=\sum_{s,t,l} \DistUL{1,2}{s,t}\PIUL{1,2,4}{s,t,l}\leq
\sum_{s,t,l} (\DistUL{1,4}{s,l}+\DistUL{2,4}{t,l})\PIUL{1,2,4}{s,t,l}=w_{1,4}+w_{2,4}.
\end{align}
Similarly, $w_{1,3}\leq w_{1,4} + w_{3,4}$, and $w_{2,3}\leq w_{2,4} + w_{3,4}$. 
Combining these inequalities and \eqref{eq:simple_proof_mid_step} we can write
$\WU{1,2,3}\leq
w_{1,2}+    w_{1,3}+w_{2,3} 
\leq(w_{1,4} + w_{2,4}) +    (w_{1,4} + w_{3,4})+ (w_{2,4} + w_{3,4})$.
Note that by construction, the bivariate marginals of $\PIU{1,2,3,4}$ in (\ref{eq:joint_probability_mass_function}) satisfy
 $\PIU{1,4}=\PIsU{(3)^{1,4}}$, $\PIU{3,4}=\PIsU{(2)^{3,4}}$, and $\PIU{2,4}=\PIsU{(1)^{2,4}}$. Hence, 
 \begin{align}
 \label{eq:simple_proof_third_step}
&\WU{1,2,3} \leq
(w^*_{1,4,(3)} + w^*_{2,4,(1)}) +  (w^*_{1,4,(3)} + w^*_{3,4,(2)})+ 
(w^*_{2,4,(1)} + w^*_{3,4,(2)}).
\end{align}
Using the
new notation, we can re-write the r.h.s. of \eqref{eq:simple_proof_first_step} as
\begin{align}
\label{eq:simple_proof_fourth_step}
&\WU{\backslash 3} + \WU{\backslash 2} + \WU{\backslash 1} =\nonumber(w^*_{1,2,(3)}+w^*_{1,4,(3)}+w^*_{2,4,(3)}) \\
&+ (w^*_{1,3,(2)}+w^*_{1,4,(2)}+w^*_{3,4,(2)}) + (w^*_{2,3,(1)}+w^*_{2,4,(1)}+w^*_{3,4,(1)}).
 \end{align}

To finish the proof we show that the
r.h.s. of \eqref{eq:simple_proof_third_step} can be upper bounded by
the r.h.s. of \eqref{eq:simple_proof_fourth_step}.
We use the triangular inequality of $w^*_{i,j,(k)}$ and apply it to the 1st, 4th, and 5th terms on the r.h.s. of \eqref{eq:simple_proof_third_step} as specified by the parenthesis:
 \begin{align*}
&(w^*_{1,4,(3)} + w^*_{2,4,(1)}) +    (w^*_{1,4,(3)} + w^*_{3,4,(2)})+ (w^*_{2,4,(1)} + w^*_{3,4,(2)})\\
&\leq ( (w^*_{1,2,(3)} +  w^*_{2,4,(3)})+  w^*_{2,4,(1)}) + 
(w^*_{1,4,(3)}+  (w^*_{1,3,(2)} +  w^*_{1,4,(2)}))\\
&+ ((w^*_{2,3,(1)} +  w^*_{3,4,(1)})+  w^*_{3,4,(2)}),\nonumber
\end{align*}
and observe that the terms in the r.h.s. of this last inequality are
accounted for on the r.h.s. of \eqref{eq:simple_proof_fourth_step}.

Note that this last step, figuring out to which terms we should apply the triangular
inequality property of $w^*$
such that we can ``cover'' the r.h.s.
of \eqref{eq:simple_proof_third_step} with the r.h.s. of \eqref{eq:simple_proof_fourth_step}, is critical. Also,
 the fact that we want to prove 
the MMOT triangle inequality holds for
$C(\NumDist) = \Theta(\NumDist)$ makes
the last step even harder.
To address this, in  the proofs of Theorems \ref{th:n_metric} and \ref{th:n_metric_c} (presented in the Appendix) we develop a general procedure and special hash functions to expand (using the triangle inequality)
and to match terms. 

%
%
\section{Numerical experiments}
\label{sec:num_exp}
We show how using a MMOT that defines an $\NumDist$-metric ($\NumDist > 2$) improves two tasks of clustering graphs compared to using a non-$\NumDist$-metric MMOT or an optimal transport (OT) that defines a $2$-metric. 
One clustering task, Section \ref{sec:exp_synthetic}, is on synthetic graphs. The other one, Section \ref{sec:exp_real}, is a real task of clustering chemical molecules.

\subsection{Multi-distance based clustering} \label{sec:multi_distance_clustering_procedure}

We use the same clustering strategy for both tasks. 
We cluster graphs by i) computing their spectrum, ii) treating each spectrum as a probability distribution, iii) using WD and three different MMOT's to compute distances among these distributions, and iv) feeding these distances to different distance-based clustering algorithms to recover the true cluster memberships, as illustrated below. We give details of each step in our procedure next.

{\centering \includegraphics[trim=0cm -0.cm 0cm 0cm, clip=true,width=1\textwidth]{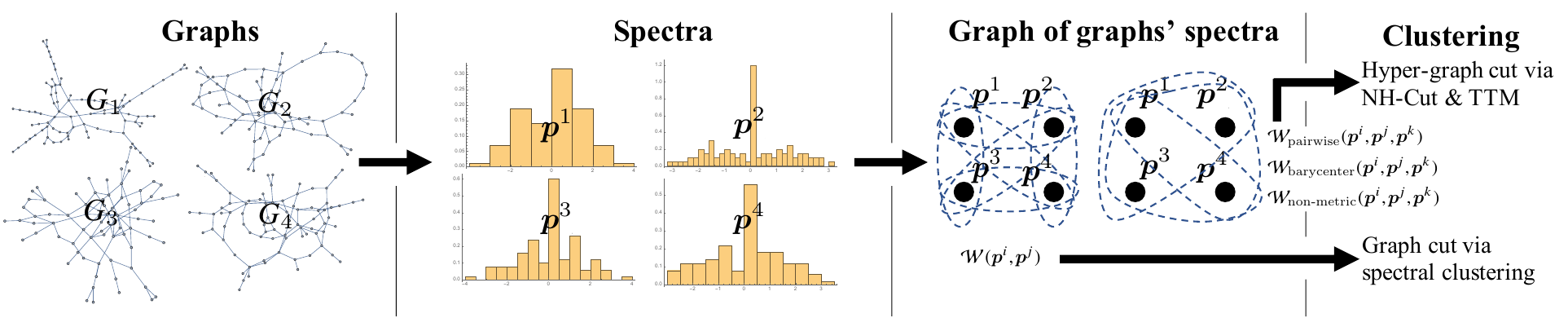}
  }

To produce the spectra, we transform the set of graphs $\{G^i\}_{i}$ into vectors $\{v^i\}_{i}$ to be clustered. Each $v^i$ is the (complex-valued) spectrum of a matrix $M^i$ representing \emph{non-backtracking walks} on $G^i$, which approximates the \emph{length spectrum} $\mu^i$ of $G^i$ \cite{torres2019non}. Object $\mu^i$ uniquely identifies (the 2-core \cite{batagelj2011fast} of)  $G^i$ (up to an isomorphism) \cite{constantine2019marked}, but is too abstract to be used directly. Hence, we use its approximation $v^i$.
The length of $v^i$ and $v^j$ for equal-sized $G^i$ and $G^j$ can be different, depending on how we approximate $\mu^i$. This motivates the use distance-based clustering and OT (multi) distances, since OT allows comparing objects of different lengths. Note that unlike the length spectrum, the classical spectrum of a graph (the eigenvalues of e.g. an adjacency matrix, Laplacian matrix, or random-walk matrix) has the advantage of having the same length for graphs with the same number of nodes. However, it does not uniquely identify a graph. For example, a star graph with $5$ nodes and the graph that is the union of a square with an isolated node are co-spectral but are not isomorphic.

To produce the (hyper) graph of graphs' spectra we need (hyper) edges and edge-weights. 
Each $v^i$ from the previous step is interpreted as a uniform distribution $\PIU{i}$ over $\XU{i} = \{v^i_k, k = 1,\mydots,\}$. 
To produce the graph of graphs' spectra, we compute a sampled version $\hat{T}^{\text{A}}$ of the matrix $T^{\text{A}} = \{\WU{i,j}\}_{i,j}$, where  $\WU{i,j}$ is the WD between $\PIU{i}$ and $\PIU{j}$ using a $\DistU{i,j}$ defined by $\DistUL{i,j}{s,t}$ $ = \lvert v^i_s - v^j_t \rvert$. 
We produce three different hyper-graph of graphs' spectra using different MMOTs.
We compute a sampled version $\hat{T}^{\text{B}}$ of the tensor $T^{\text{B}} = \{\WU{i,j,k}\}_{i,j,k}$, where $\WU{i,j,k}$ is our pairwise MMOT defined as in Def. \ref{def:pair_wise_MMOT} with $\DistU{i,j}$ defined as for $T_{\text{A}}$.
We also compute a sampled version $\hat{T}^{\text{C}}$ of the tensor $T^{\text{C}} = \{\WU{i,j,k}\}_{i,j,k}$, where $\WU{i,j,k}$ is the barycenter MMOT defined as in Remark \ref{rmk:barycenter} with the $\DistU{i,j}$ as before.
Finally, we compute a sampled version $\hat{T}^{\text{D}}$ of the tensor $T^{\text{D}}$ with $T^{\text{D}}_{i,j,k} =\WU{i,j,k}$, where $\WU{i,j,k}$ is the non-$n$-metric defined in Theorem \ref{app:th:counter_example}, but now considering points in the complex plane.
The sampled tensors $\hat{T}^{\text{B}}$, $\hat{T}^{\text{C}}$, and $\hat{T}^{\text{D}}$ are built by randomly selecting $z$ triples $(i,j,k)$, $z=100$ in the synthetic experiments and $z=600$ in the real experiments, and setting $\hat{T}^{\text{B}}_{i,j,k} = {T}^{\text{B}}_{i,j,k}$, $\hat{T}^{\text{C}}_{i,j,k} = {T}^{\text{C}}_{i,j,k}$, and $\hat{T}^{\text{D}}_{i,j,k} = {T}^{\text{D}}_{i,j,k}$. The non-sampled triples are given a very large value. The sampled matrix $\hat{T}^{\text{A}}$ is built by sampling $(3/2)\times z$ pairs $(i,j)$ and setting  $\hat{T}^{\text{A}}_{i,j} = {T}^{\text{A}}_{i,j}$, and setting a large value for non-sampled pairs. All (multi) distances $\WU{i,j}$ and $\WU{i,j,k}$ amount to solving a linear program, for which we use CVX \cite{cvx14,gb08}.

To obtain clusters, we feed the weighted (hyper) graphs specified by $\hat{T}^{\text{A}}$, $\hat{T}^{\text{B}}$, $\hat{T}^{\text{C}}$, and $\hat{T}^{\text{D}}$, to different (hyper) graph-cut algorithms.
We feed the matrix of distances $\hat{T}^{\text{A}}$ to a spectral clustering algorithm \cite{shi2000normalized} based on normalized random-walk Laplacians to produce one clustering solution, which we name $\mathcal{C}^{\text{A}}$.
We feed the multi-distances $\hat{T}^{\text{B}}$, $\hat{T}^{\text{C}}$ and $\hat{T}^{\text{D}}$ to the two hypergraph-based clustering methods NH-Cut \cite{ghoshdastidar2017consistency} and TTM \cite{ghoshdastidar2017uniform,ghoshdastidar2015provable}.
These produce different clustering solutions which we name $\mathcal{C}^{\text{B1}}$ and $\mathcal{C}^{\text{B2}}$,  $\mathcal{C}^{\text{C1}}$ and $\mathcal{C}^{\text{C2}}$, and  $\mathcal{C}^{\text{D1}}$ and $\mathcal{C}^{\text{D2}}$ respectively. 
Both NH-Cut and TTM find clusters by (approximately) computing minimum cuts of a hypergraph where each hyperedge's weight is the MMOT distance among three graphs.
Both NH-Cut and TTM require a threshold that is used to prune the hypergraph. Edges whose weight (multi-distance) is larger than a given threshold are removed. This threshold is tuned to minimize each  clustering solution error.
We code all clustering algorithms to output $N$ clusters, the correct number of clusters.

To compute the quality of each clustering solution, we compute the fraction of miss-classified graphs. In particular, if for each  $x \in \{ $ A, B1, B2,C1, C2, D1, D2 $\}$, we define $\mathcal{C}^x(i) = k$ to mean that clustering solution $\mathcal{C}^x$ assigns graph $G^i$ to cluster $k$, then the error in solution $\mathcal{C}^x$ is
\begin{equation} \label{eq:app:formula_for_clustering_error}
\min_{\sigma} \frac{1}{N}\sum^{N}_{i=1} \mathbb{I}(\mathcal{C}^{\text{ground truth}}(i) = \mathcal{C}^{x}(\sigma(i)),
\end{equation}
where $N$ is the number of clusters, and the $\min$ over all permutations $\sigma$ of the elements $\{1,\dots,N\}$ is needed because the specific cluster IDs output have no real meaning. 
This value is computed $100$ times (with all random numbers in our code being drawn independently among experiments) so that we can report an histogram of errors, or error bars, for each method.

The complete code for our experiments can be found here \url{https://github.com/bentoayr/MMOT}.

%
%
\subsection{Synthetic graphs dataset}\label{sec:exp_synthetic}
%
We generate seven synthetic clusters of ten graphs each by including in each cluster multiple random perturbations 
-- with independent edge addition/removal with $p=0.05$ -- of a complete graph, a complete bipartite graph,  a cyclic chains,  a $k$-dimensional cube,  a $K$-hop lattice,  a periodic 2D grid, or an Erdős–Rényi graph. 
A random class prediction has a $0.857$ error rate. 
To recover the class of each graph, we follow the procedure described in Section \ref{sec:multi_distance_clustering_procedure}. 
In Figure \ref{fig:num_res} we show the distribution of clustering errors using different (multi)-distances and clustering algorithms over
 $100$ independent experiments. The mean error rates are written in the legend box of each figure.
\begin{figure}[h!]
\centering
\includegraphics[width=0.32\textwidth,trim={0.65cm 0.cm 1.2cm 0.1cm},clip]{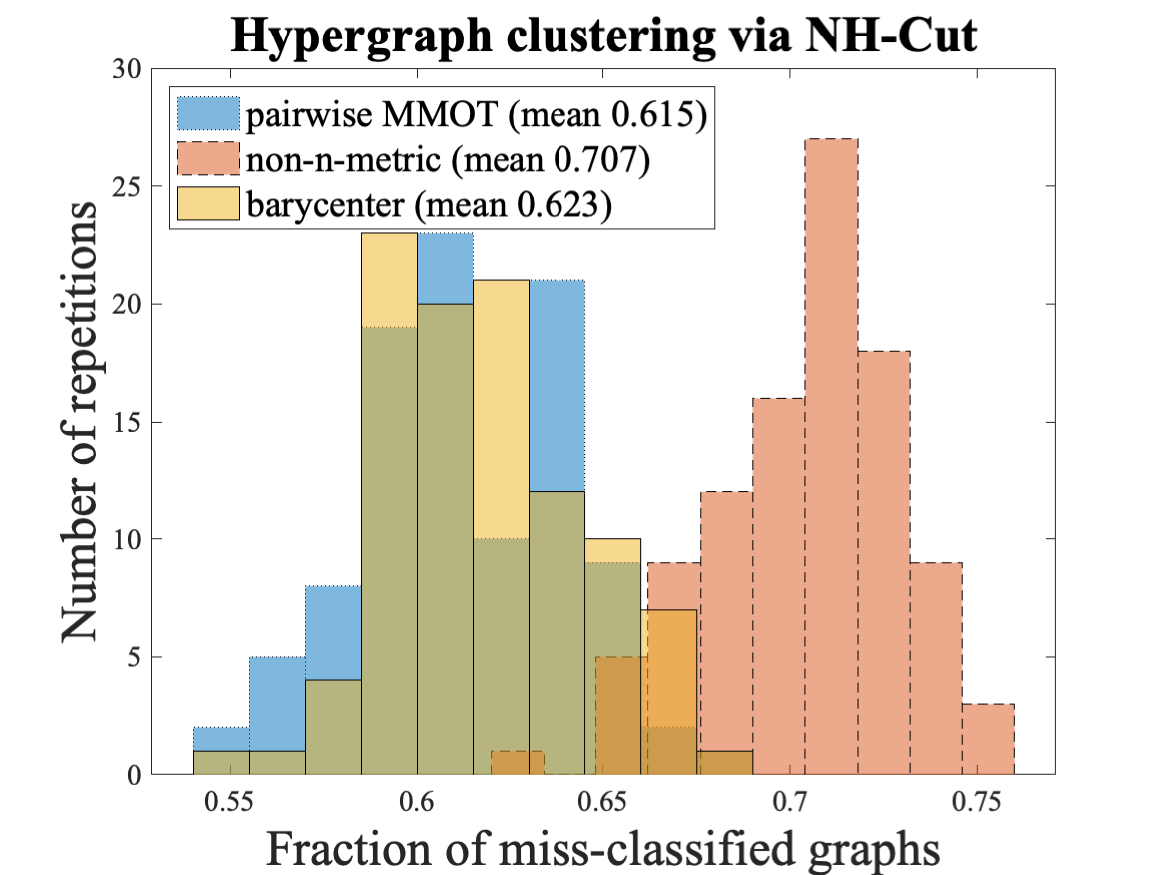}
\includegraphics[width=0.32\textwidth,trim={0.65cm 0.cm 1.2cm 0.1cm},clip]{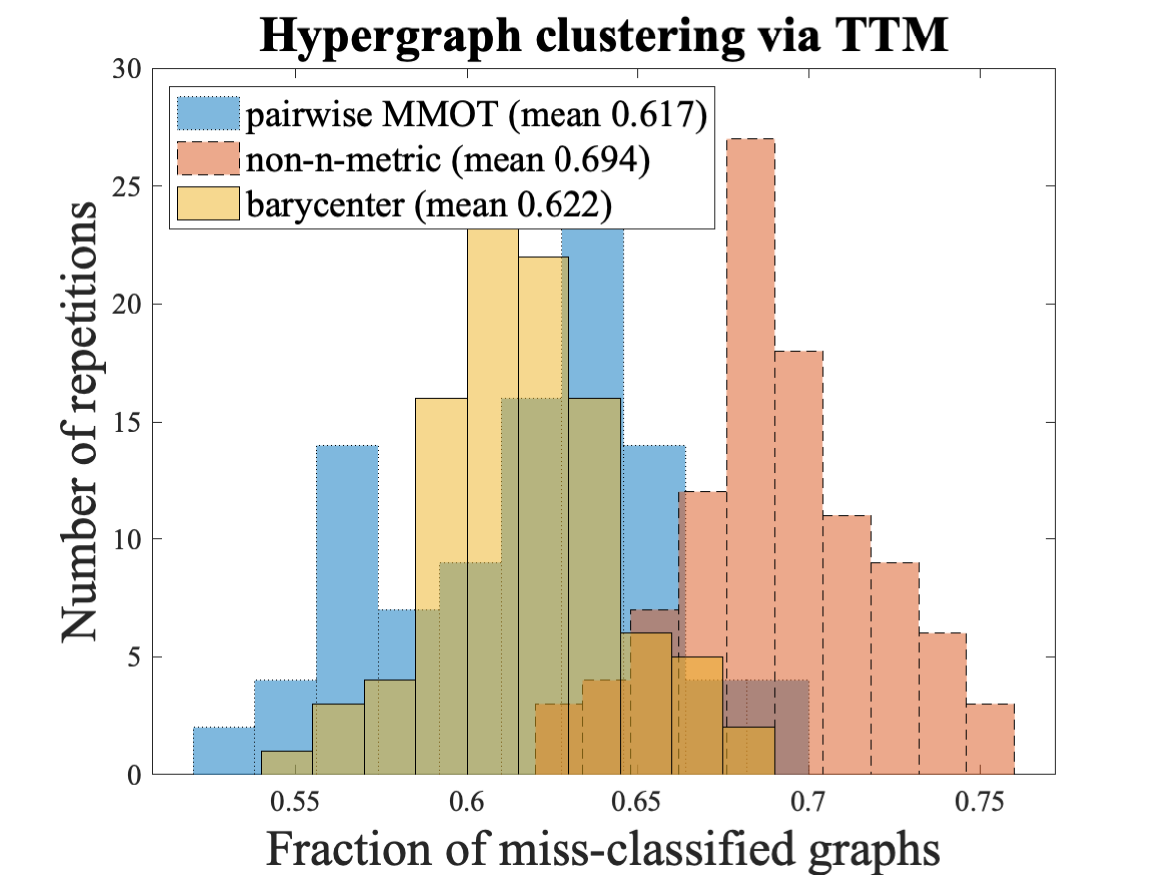}
\includegraphics[width=0.32\textwidth,trim={0.65cm 0.cm 1.2cm 0.1cm},clip]{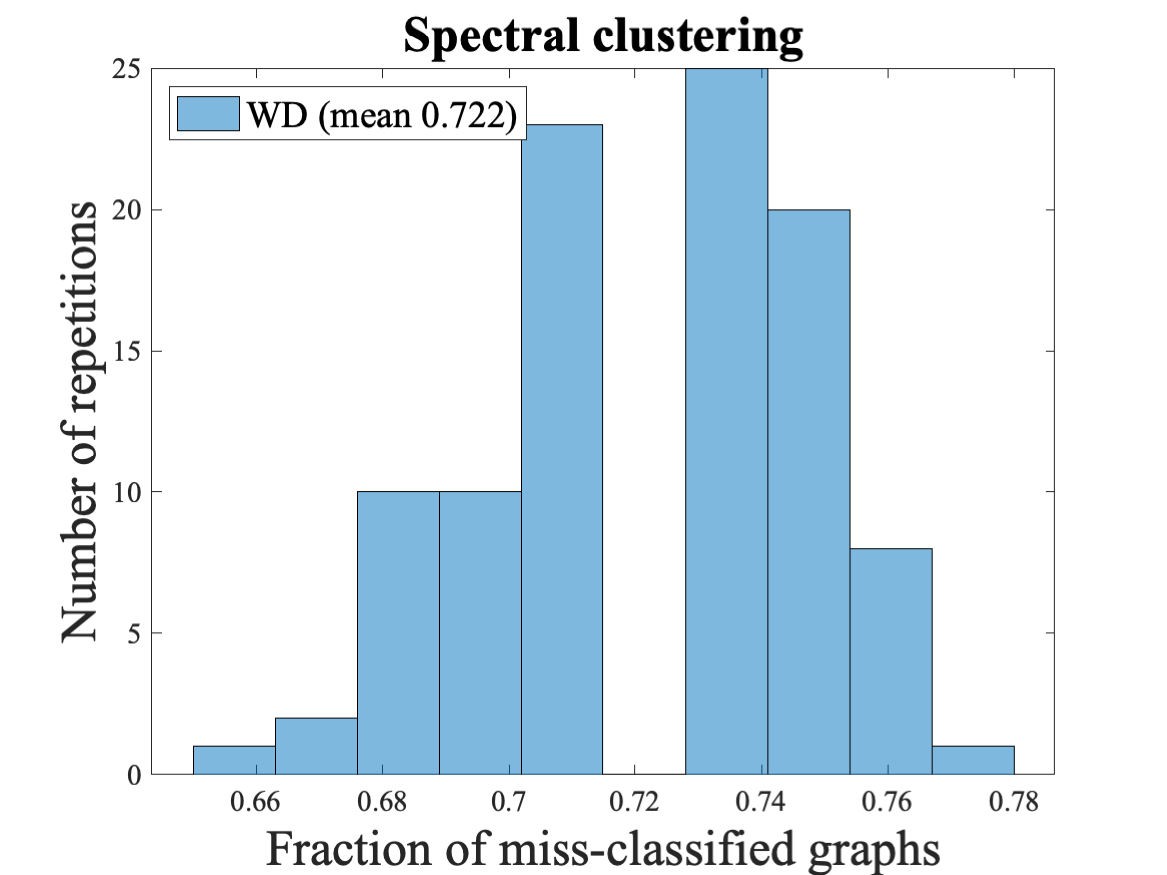}
%
%
\caption{\small Comparing the effect that different distances and metrics have on clustering \textit{synthetic graphs}. Histogram with dashed outline is for non-$n$-metric, with solid outline for barycenter, and with dotted outline for pairwise MMOT.
} 
\label{fig:num_res}
\end{figure}

Figures \ref{fig:num_res}-(left, center) show that, as expected,  both TTM and NH-Cut work better when hyperedges are computed using an $\NumDist$-metric, and that pairwise MMOT works better than WBD. 

Figure \ref{fig:num_res}-(right) shows that clustering using only pairwise relationships among graphs leads to worse accuracy than if using triple-wise relationships as in Figures \ref{fig:num_res}-(left, center). %
The results in Figure \ref{fig:num_res}-(right) are similarly bad to the results one obtains when use the MMOT explained in Remark \ref{rmk:comment_on_sum_of_inde_pairwise_sum}, which does not give joint transports and is trivially an $\NumDist$-metric. Indeed, both TTM and NH-Cut are based on spectral clustering of the hypergraph weight tensor by first reducing it to a matrix via ``summing out'' one of the dimensions and then using regular spectral clustering techniques. At the same time, if tensor a $T_{i,j,k}$ satisfies $T_{i,j,k} = \W_{i,j} + \W_{j,k} + \W_{i,k}$ for some distances $\W_{i,j}$, the matrix obtained via such reduction has spectral properties close to those of $\W_{i,j}$.

\subsubsection{Injection of triangle inequality violations} \label{app:inject_TIVs}

To double check that this difference in performance is due to the $\NumDist$-metric properties of pairwise MMOT and WBD, we perturb $\Wdist{}$ to introduce triangle inequality violations
and measure its effect on clustering accuracy.

To introduce triangle inequality violations, we perturb the tensor $\Wdist{i,j,k}$ as follows. For each set of four different graphs $(i,j,k,l)$, we find which among $\Wdist{i,j,k},\Wdist{i,j,l},\Wdist{i,l,k},\Wdist{l,j,k}$ we can change the least to produce a triangle inequality violation among these values. Let us assume that this value is $\Wdist{i,j,k}$, and that to violate $\Wdist{i,j,k} \leq \Wdist{i,j,l}  + \Wdist{i,l,k} + \Wdist{l,j,k}$, the quantity $\Wdist{i,j,k}$ needs to increase at least by $\delta$, where $\delta = \Wdist{i,j,l}  + \Wdist{i,l,k} + \Wdist{l,j,k}  - \Wdist{i,j,k}$. We then increase $\Wdist{i,j,k}$ by $1.3\times \delta$.
We repeat this procedure such that, in total, $20\%$ of the entries of the tensor $\Wdist{}$ get changed.

Table \ref{tab:adding_violations} shows the effect of adding
violations on the mean error rate for different MMOT distances.
These violations clearly affect pairwise MMOT and barycenter-MMOT (both $\NumDist$-metrics), but do not have an impact on  non $n$-metric distances. 
\begin{table}[!h]
\setlength\tabcolsep{1.5pt} 
\centering
\footnotesize{
\begin{tabular}{|c||c|c|c|c||c|c|c|c|}
\hline
\textbf{With violations?} & \textbf{Clustering} & \textbf{Pairwise} & \textbf{WBD} & \textbf{Non-$\NumDist$-metric}  & \textbf{Clustering} & \textbf{Pairwise} & \textbf{WBD} & \textbf{Non-$\NumDist$-metric} \\ \hline
\textbf{No}         & \textbf{NH-Cut}     & 0.615             & 0.623               & 0.707                         & \textbf{TTM}        & 0.617             & 0.622               & 0.694             \\ \hline
\textbf{Yes}        & \textbf{NH-Cut}     & 0.632             & 0.632               & 0.704                       & \textbf{TTM}        & 0.627             & 0.634               & 0.696                 \\ \hline
\end{tabular}
}
\caption{\small The mean error rates for clustering of synthetic graph datasets for different MMOT metric and non-metric distances, with and without injected  triangle inequality violation, are shown. Triangle inequality
violations in $\Wdist{}$ degrades
clustering performance with $\NumDist$-metrics more than with non-$\NumDist$-metrics.
\label{tab:adding_violations}
}
\end{table}

\subsubsection{Reproducibility}

Our code is fully written in Matlab 2020a. It requires installing CVX, available in \url{http://cvxr.com/cvx/download/}. 

To produce Figure \ref{fig:num_res} open Matlab and run the file \verb!run_me_for_synthetic_experiments.m!. To produce the numbers in the second row of  Table \ref{tab:adding_violations}, run the same file but with \verb!fraction_viol = 0.2; strength_of_viol = 0.3;!. Note that the numbers in the first row of Table \ref{tab:adding_violations} are the mean values in Figure \ref{fig:num_res}. The call to \verb!run_me_for_synthetic_experiments.m! takes $26$ hours to complete using 12 cores, each core from an Intel(R) Xeon(R) Processor E5-2660 v4 (35M Cache, 2.00 GHz). Except for the \verb!Weiszfeld.m! file, all of the code
was written by us and is distributed under an MIT License. This license is described in the \verb!README.txt! file at the root our repository. The license for  \verb!Weiszfeld.m! is on the header of the file itself.

\subsection{Molecular graphs dataset} \label{sec:exp_real}
This experiment is motivated by the important task in chemistry of clustering chemical compounds, represented as graphs, by their structure \cite{wilkens2005hiers,seeland2014structural,mcgregor1997clustering}.
We use the molecular dataset in the supplementary
material of \cite{sutherland2003spline}, which can be downloaded
at \url{https://pubs.acs.org/doi/abs/10.1021/ci034143r#_i21}.
It contains the adjacency
matrices of graphs corresponding to five types of compounds: cyclooxygenase-2  inhibitors ($467$ graphs), benzodiazepine receptor ligands ($405$ graphs), estrogen receptor ligands ($1009$ graphs), dihydrofolate reductase inhibitors ($756$ graphs), and monoamine oxidase inhibitors ($1641$ graphs). 

To build our clusters, we randomly
get ten graphs of each type, and prune them
so that they have no node with a degree smaller than $1$. We note that, unlike for the synthetic
data, the molecular graphs have weighted
adjacency matrices, whose entries can be $0,1$, or $2$. 
A random class prediction has $0.8$ error rate. 
To recover the class of each graph, we follow the procedure described in Section \ref{sec:multi_distance_clustering_procedure}. 
We repeat this experiment $100$ times, independently, to collect error statistics.

In Figure \ref{fig:num_res_real_graphs} 
we show the distribution of clustering errors using different (multi)-distances and clustering algorithms.
\begin{figure*}[h!]
\centering
\includegraphics[width=0.32\textwidth,trim={0.65cm 0.cm 1.2cm 0.1cm},clip]{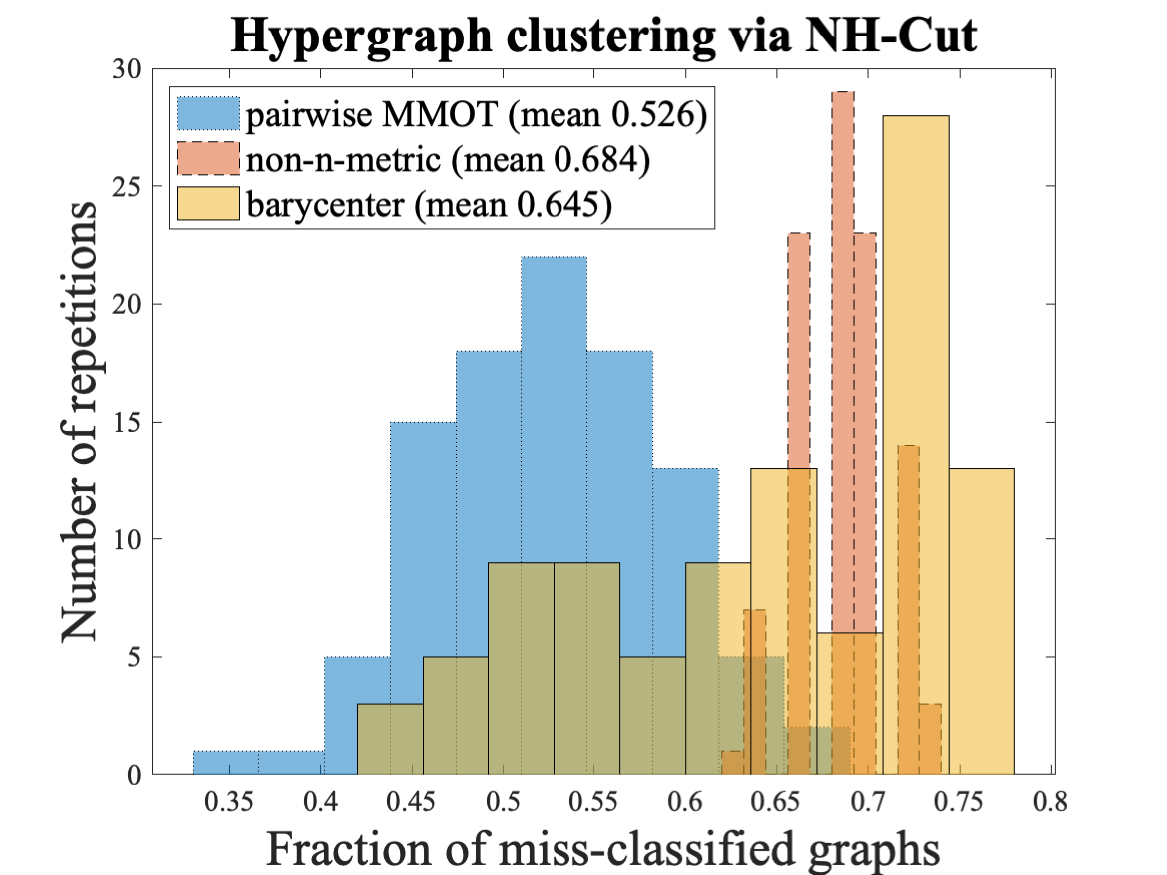}
\includegraphics[width=0.32\textwidth,trim={0.65cm 0.cm 1.2cm 0.1cm},clip]{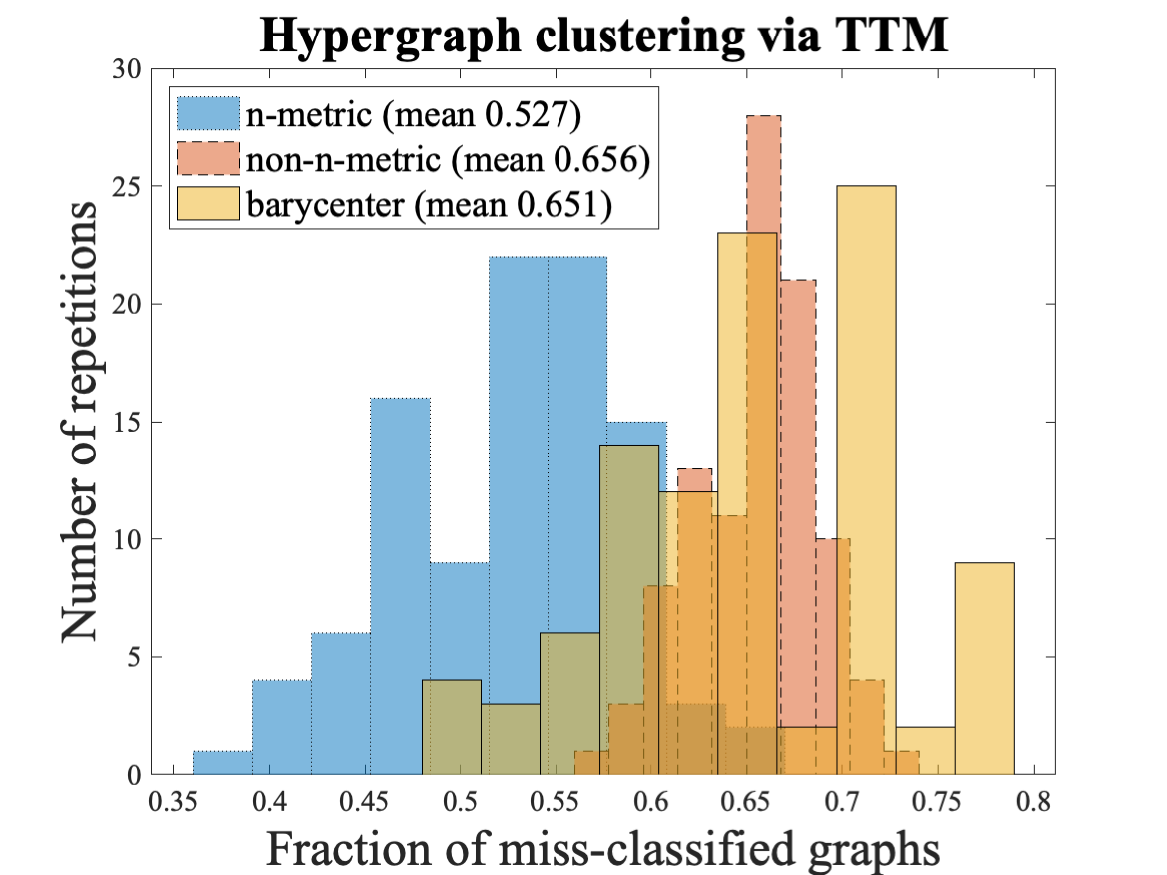}
\includegraphics[width=0.32\textwidth,trim={0.65cm 0.cm 1.2cm 0.1cm},clip]{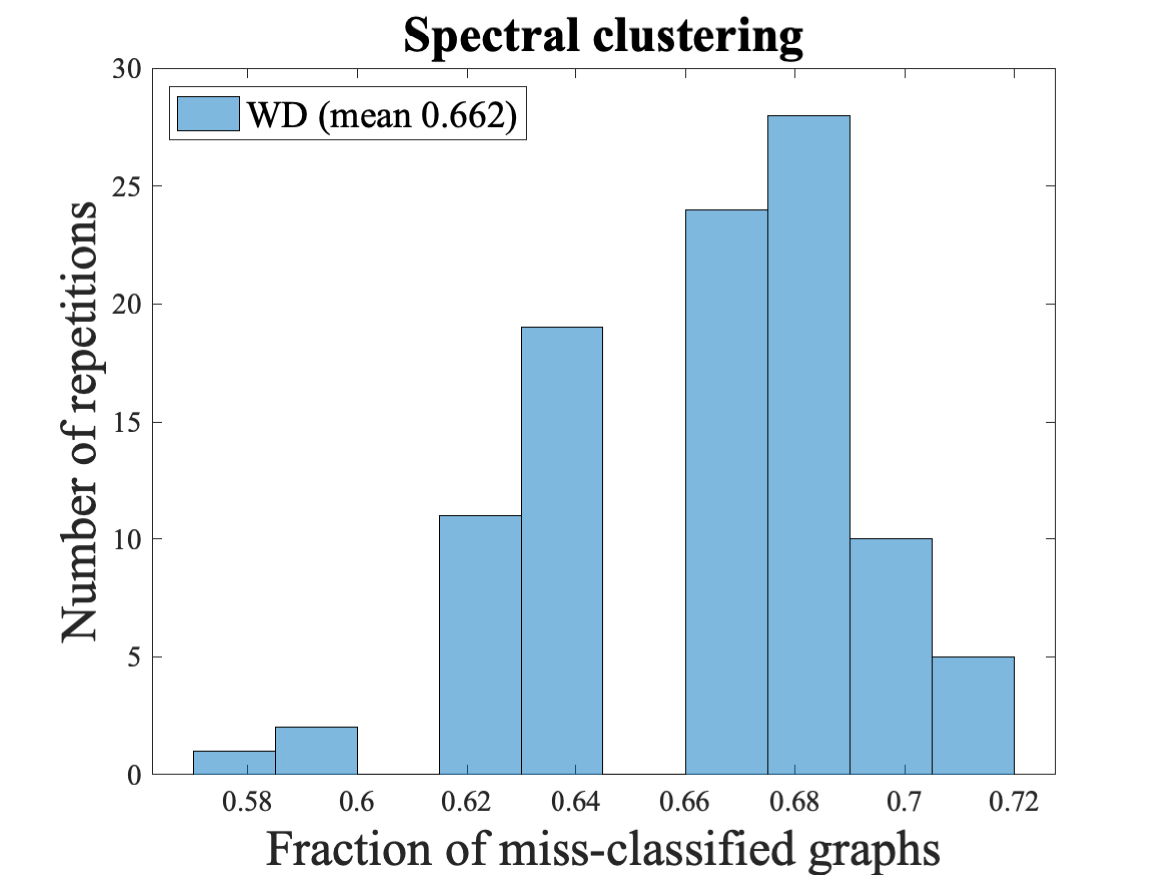}
%
%
\caption{\small Comparing the effect that different distances and metrics have on clustering \textit{molecular graphs}. Histogram with dashed outline is for non-$n$-metric, with solid outline for barycenter, and with dotted outline for pairwise MMOT.}
\label{fig:num_res_real_graphs}
\end{figure*}
Figures \ref{fig:num_res_real_graphs}-(left, center) show  that both TTM and NH-Cut work better when hyperedges are computed using $\NumDist$-metrics, and Figure \ref{fig:num_res_real_graphs}-(right) shows that clustering using pairwise relationships performs worse than using triple-wise relations. For the same reasons as explained in Section \ref{sec:exp_synthetic}, the results in Figure \ref{fig:num_res_real_graphs}-(right) are similarly bad to the results one would obtain if we used the MMOT explained in Remark \ref{rmk:comment_on_sum_of_inde_pairwise_sum}.

There is a starker difference between $\NumDist$-metrics and
 non-$\NumDist$-metrics than that seen in Figure \ref{fig:num_res_real_graphs}, which we now discuss.
 The number of possible $3$-sized hyperedges is cubic with the number of graphs being clustered. Thus, in our experiments we randomly sample $z$ triples $(i,j,k)$ and only for these we create an hyperedge with weight $\Wdist{i,j,k}$. 
 Figure \ref{fig:clustering_real_molecules} shows the effect of $z$ ($x$-axis) on performance. Comparing more graphs, i.e. increasing $z$, should improve clustering. However, for a non-$\NumDist$-metric, 
 as $z$ grows, triangle inequality violations can appear that introduce confusion: a graph can be ``close'' to two clusters that are far away, confusing TTM and NH-Cut. This compensates the benefits of a high $z$ and results in the flat curves in Figure \ref{fig:clustering_real_molecules}.

\begin{figure}[h!]
 \centering
  {\includegraphics[trim=0cm -0.cm 0cm 0cm, clip=true,width=0.6\textwidth]{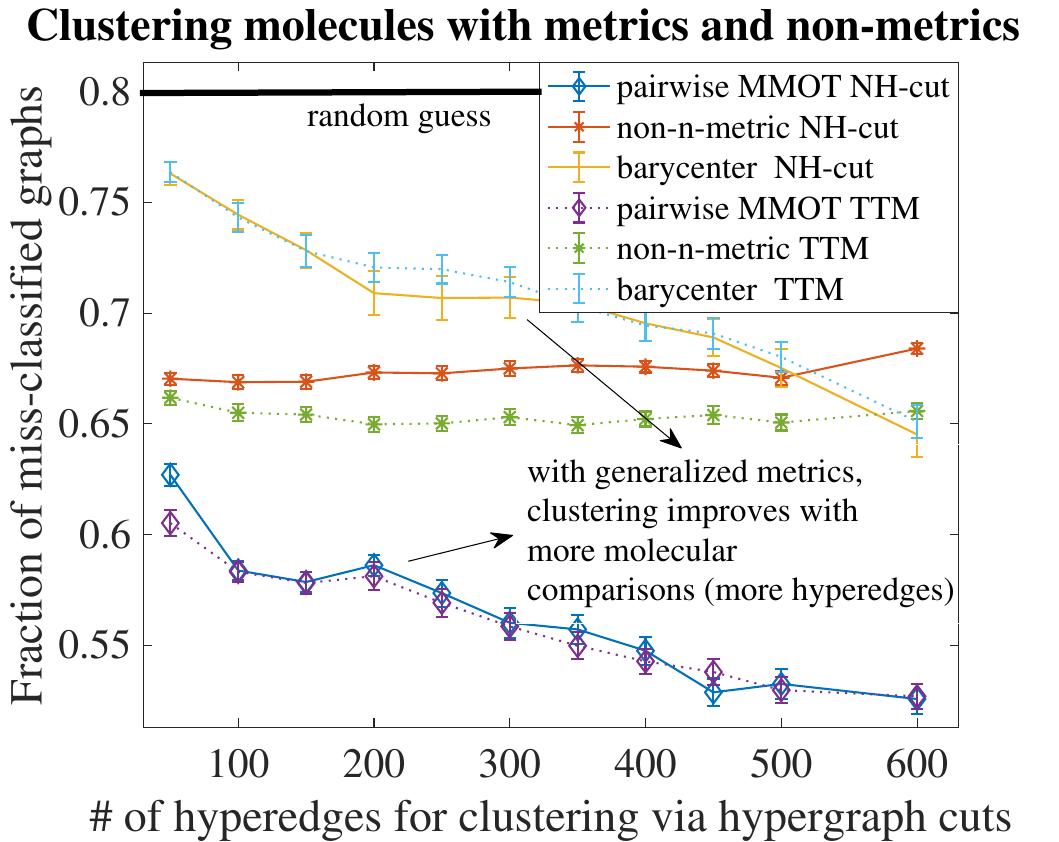}
  }
  \caption{\small  Hypergraphs are built by comparing molecules' shapes using MMOT and creating hyperedges for similarly-shaped molecule groups. With a non generalized metric, gains from using a richer hypergraph (more hyperedges) are lost due to anomalies allowed by the lack of metricity (flat curves in green and red), not so with a generalized metric (negative slope curves). The generalised metric studied (pairwise) leads to a better clustering than previously studied metrics (barycenter).  Error bars are the standard deviation of the average over $100$ independent experiments.
  }
  \label{fig:clustering_real_molecules}
\end{figure}

 \subsubsection{Reproducibility}

To produce Figure \ref{fig:num_res_real_graphs}   open Matlab and run the file \verb!run_me_for_molecular_experiments.m!. The call to \verb!run_me_for_molecular_experiments.m! takes $31$ hours to complete using 12 cores, each core from an Intel(R) Xeon(R) Processor E5-2660 v4 (35M Cache, 2.00 GHz). Except for the \verb!Weiszfeld.m! file and the molecular dataset, all of the code
was written by us and is distributed under an MIT License. This license is described in the \verb!README.txt! file at the root our repository. The license for \verb!Weiszfeld.m! is on the header of the file itself. The license for the dataset is described on the \verb!README.txt! file inside the dataset folder. 
 
We contacted that authors (via email) about the use of their dataset and they have informed
us that there are no licenses attached to it, as long as we attribute it a citation. This data contains no personally identifiable information nor offensive content.

\section{Discussion and future work}\label{sec:future}
%
In this paper, we have proved that for a general MMOT, the cost function being a generalized metric is not sufficient to guarantee that MMOT defines a generalized metric. Nevertheless, we have shown that a new family of multi-distances that generalize optimal transport to multiple distributions, the family of pairwise multi-marginal optimal transports (pairwise MMOT), leads to a multi-distance that satisfies generalized metric properties. 
This now opens the door to us using pairwise MMOT in combination with several algorithms whose good performance depends on metric properties.
In addition, we have established coefficients for the generalized triangle inequality associated with the pairwise MMOT, and  proved that these coefficients cannot be improved, up to a linear factor. 

Our results are for the pairwise MMOT family. 
In future work, we seek to find new sufficient conditions under
which other variants of MMOT lead to generalized metrics,
and, for certain families of MMOT, find necessary conditions
for these same properties to hold.

Finally, in future work we will also study how the  structure of optimal coupling among distributions induced by our MMOTs, i.e. the support of $\PIs$, compares with those of other MMOTs, such as the Barycenter MMOT.

\section*{Declarations}

\subsection*{Funding}

The authors gratefully acknowledge the support of the National Science Foundation (IIS-1741129) and the National Institutes of Health (grant 1U01AI124302).

\subsection*{Conflict of interest/Competing interests}

Not applicable

\subsection*{Ethics approval}

Not applicable

\subsection*{Consent to participate}

Not applicable

\subsection*{Consent for publication}

\subsection*{Availability of data and materials}

All used data is available via the link below.

\url{https://github.com/bentoayr/MMOT/Datasets}

\subsection*{Code availability} 

All used code is available via the link below.

\url{https://github.com/bentoayr/MMOT}

\subsection*{Authors' contributions}

Jos\'e Bento, Azadeh Sheikholeslami, and Liang Mi contributed equally in effort among the tasks of writing of the manuscript, the checking of the correctness of the proofs, and the running of the numerical experiments.
Liang Mi made the first pass on the writing of paper.
Jos\'e Bento set the goals of the project.

\noindent

\bibliography{main}

\begin{appendices}

\newpage
\clearpage

\section{Details for proof of Theorem \ref{app:th:counter_example}}\label{app:proof_details_of_counter_example}

\begin{proof}
Note that Definition \ref{def:gen_metric_d} supports using a different function $\DistU{i,j,k}$ for different product sample spaces $\XU{i}\times\XU{j}\times\XU{k}$. In the case of Theorem \ref{app:th:counter_example}, however, we only use $\X\times\X\times\X$, so, when checking the $\NumDist$-metric properties, we can drop the upper indices in $\Dist$ in Definition \ref{def:gen_metric_d}.

For simplicity, we will abuse the notation and use $\Dist(x,y,z)$ and $\DistUL{}{i,j,k}$ interchangeably, where $i$, $j$, and $k$ are the index of $x, y$, and $z$, in $\XU{}$.

Given $x,y,z,w \in \XU{}$, it is immediate to see that (i) $\Dist(x,y,z) \geq 0$, (ii) $\Dist(x,y,z)$ is permutation invariant, and that (iii)  $\Dist(x,y,z) = 0$ if and only $x=y=z$ (remember that there are no three co-linear points in $\XU{}$). 
It is also not hard to see that, $\Dist(x,y,z) \leq \Dist(x,y,w)  + \Dist(x,w,z) + \Dist(w,y,z)$. To be specific, if $\Dist(x,y,z) = 0$, then the inequality is obvious. If $\Dist(x,y,z) = \gamma$, then without loss of generality we can assume that $x = y \neq z$. In this case, if furthermore $w = x$, then $\Dist(x,w,z) = \gamma$, and the inequality follows. If $w = z$, then $\Dist(x,y,w) = \gamma$, and the inequality follows. If $w$ is different from $x,y,z$ then $\gamma \leq \Dist(x,w,z)$, and the inequality follows.
If $\Dist(x,y,z) > \gamma$, it must be that $x,y$ and $z$ are different. In which case we need do consider two special cases. If $w$ is equal to one among $x,y,z$, say $w=x$ without loss of generality, then $\Dist(x,y,z) = \Dist(y,z,w)$, and the inequality follows. If $w$ is different from $x,y,z$, then we have $\Dist(x,y,z) = \Dist(x,y,w)  + \Dist(x,w,z) + \Dist(w,y,z)$ if $w$ is contained by the triangle formed by $x, y$, and $z$, and otherwise, we have 
$\Dist(x,y,z) < \Dist(x,y,w)  + \Dist(x,w,z) + \Dist(w,y,z)$. In other words, $\Dist$ is an $\NumDist$-metric ($\NumDist = 3$).

Given a mass function $\PIU{i,j,k}$, the value $\inner{\DistUL{i,j,k}{}}{\PIU{i,j,k}}$ represents the average area of the triangle whose three vertices are sampled from $\PIU{i,j,k}$. Computing the MMOT distance $\Wdist{i,j,k}$ for the mass functions $\PIU{i}$, $\PIU{j}$, $\PIU{k}$, amounts to finding the mass function $\PIsU{i,j,k}$ with univariate marginals $\PIU{i}$, $\PIU{j}$, $\PIU{k}$ that minimizes this average area.

Now consider $\PIU{1}$, $\PIU{2}$, $\PIU{3}$, and $\PIU{4}$ as depicted in Figure \ref{app:fig:counter_example}-(left). The mass functions $\PIU{1}$, $\PIU{2}$ assign probability one to each one of the blue and red points, and zero probability to every other point in $\XU{}$. The mass functions $\PIU{3}$ and $\PIU{4}$ assign equal probability to each one of the green points, and orange points, respectively, and $0$ probability to other points in $\X$.

Now we compute the distances $\Wdist{1,2,3}$, $\Wdist{1,2,4}$, $\Wdist{1,3,4}$, and $\Wdist{2,3,4}$.
The MMOT distance $\Wdist{1,2,3}$ is equal to the average of the area of the two shaded triangles in Figure \ref{app:fig:counter_example_W123_W_124}-(left), which is $\Wdist{1,2,3} = 0.5\times(0.5) + 0.5\times(0.5) = 0.5$. The distance $\Wdist{1,2,4}$ is equal to the average of the area of the two shaded triangles in Figure \ref{app:fig:counter_example_W123_W_124}-(right), which is $\Wdist{1,2,4} = 0.5\times (0.5 \epsilon ) + 0.5\times (0.25 - 0.5\epsilon ) = 0.125$.
\begin{figure}[h]
\centering
\includegraphics[scale=0.33, trim={0cm 0cm 0cm 0cm}, clip]{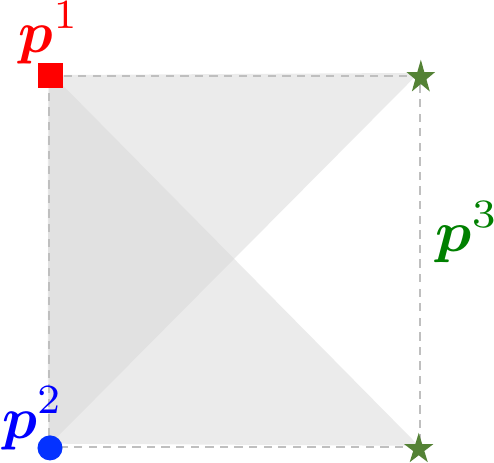}\quad\quad\quad
\includegraphics[scale=0.33, trim={0cm 0cm 0cm 0cm}, clip]{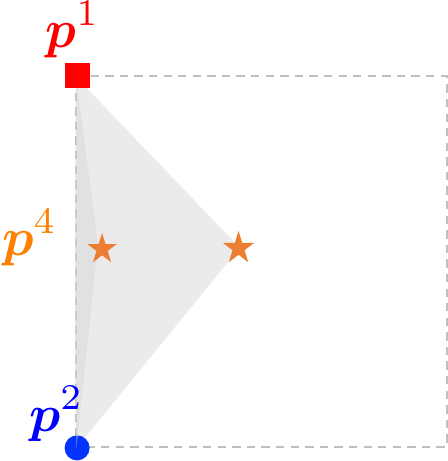}
\caption{\small (Left) Triangles associated with the optimal distribution of triples $\PIsU{1,2,3}$ associated with $\WU{1,2,3}$. Red square is ${\bm p}^1$. Blue circle is ${\bm p}^2$.  Green starts are ${\bm p}^3$. (Right) Triangles associated with the optimal distribution $\PIsU{1,2,4}$ associated with $\WU{1,2,4}$. Red square is ${\bm p}^1$. Blue circle is ${\bm p}^2$.  Yellow starts are ${\bm p}^4$.} 
\label{app:fig:counter_example_W123_W_124}
\end{figure}

The MMOT distances for $\Wdist{1,3,4}$ and $\Wdist{2,3,4}$ are the same by symmetry. We focus on the computation of $\Wdist{2,3,4}$. Since both $\PIU{4}$ and $\PIU{3}$ are uniform over their respective supports, it must be the case that $\PIsU{2,3,4}$ -- the optimal joint distribution in the computation of $\Wdist{2,3,4}$ -- has a bi-variate marginal $\PIsU{3,4}$ of the form
$$\{\{\PIsUL{3,4}{1,1},\PIsUL{3,4}{2,1}\},\{\PIsUL{3,4}{1,2},\PIsUL{3,4}{2,2}\}\} =\hspace{-0.08cm} \left\{\left\{ \alpha ,\frac{1}{2} \hspace{-0.08cm}-\hspace{-0.08cm} \alpha \right\},\left\{ \frac{1}{2} \hspace{-0.08cm}-\hspace{-0.08cm} \alpha, \alpha \right\}\right\},$$
where $\alpha \in \left[0, \frac{1}{2}\right]$. Therefore, the distance $\Wdist{2,3,4}$ is equal to the weighted average of the area of the four shaded triangles in Figure \ref{app:fig:counter_example_W234}, where we split the four triangles into two different drawings for clarity sake.
\begin{figure}[h]
\centering
\includegraphics[scale=0.33, trim={0cm 0cm 0cm 0cm}, clip]{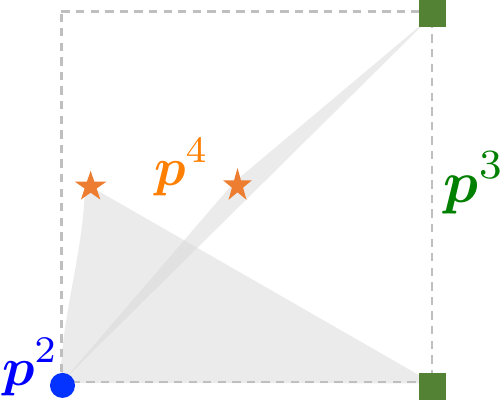}\quad\quad\quad
\includegraphics[scale=0.33, trim={0cm 0cm 0cm 0cm}, clip]{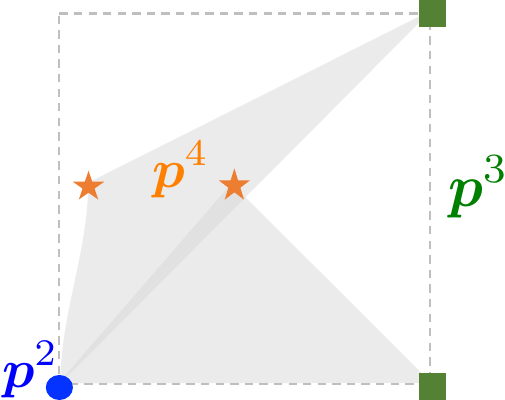}
\caption{\small Triangles associated with the optimal distribution of triples $\PIsU{2,3,4}$ associated with $\WU{2,3,4}$. Green squares are ${\bm p}^3$. Blue circle is ${\bm p}^2$.  Yellow starts are ${\bm p}^4$.} 
\label{app:fig:counter_example_W234}
\end{figure}
In other words, 
\begin{align*}
     \WU{2,3,4} & = \min_{\alpha\in[0,\frac{1}{2}]} \Big\{  \alpha\left(\frac{1}{4} \right) + \alpha\left(\frac{\epsilon}{2}\right) +  \left(\frac{1}{2} - \alpha\right)\left(\frac{1}{4} - \frac{\epsilon}{2}\right)  + \left(\frac{1}{2} - \alpha\right)\left(\frac{1}{4}\right)  \Big\} \\
&= \min \Big\{  \frac{1}{2}\left(\frac{1}{4} \right) + \frac{1}{2}\left(\frac{\epsilon}{2}\right), \frac{1}{2}\left(\frac{1}{4} - \frac{\epsilon}{2}\right) + \frac{1}{2}\left(\frac{1}{4}\right)  \Big\} = \frac{1}{8} + \frac{\epsilon}{4},
\end{align*}
where we are using the fact that $\min_{\alpha\in[0,\frac{1}{2}]}(\text{linear function of } \alpha)$ must be minimized at either extreme $\alpha = 0$ or $\alpha = \frac{1}{2}$.
It is finally easy to observe that 
\begin{align*}
    \mathcal{W}^{1,2,3} = \frac{1}{2} > \frac{1}{8} + \left(\frac{1}{8} + \frac{\epsilon}{4} \right)+ \left(\frac{1}{8} + \frac{\epsilon}{4}\right) = \mathcal{W}^{1,2,4} + \mathcal{W}^{1,3,4} + \mathcal{W}^{2,3,4}.
\end{align*}
\end{proof}

\section{Useful lemmas} \label{sec:useful_lemmas}
\begin{lemma}\label{th:marginals_of_G_map}
Let $\PI$ be as in Def. \ref{def:new_gluing_map} eq. \eqref{eq:def_of_G_1} for some $\QIU{k}$ and $\{\QIU{i\mid k}\}_{i\in[\NumDist]\backslash k}$.
Let $\PIU{i}$ and $\PIU{i,k}, i \neq k,$ be the marginals of $\PI$ over $\XU{i}$ and $\XU{i} \times \XU{k}$, respectively.
Let $\QIU{i,k} = \QIU{i\mid k}\QIU{k}$, $i \neq k$, and let $\QIU{i}$ be its marginals over $\XU{i}$. We have that $\PIU{i} = \QIU{i} \; \forall i,$ and $\PIU{i,k} = \QIU{i,k} \; \forall i \neq k$.
\end{lemma}
\begin{proof}
Think of $\PI$ as describing $\NumDist$ discrete random variables (r.v.'s). 
It follows from the factorisation in \eqref{eq:def_of_G_1} that conditioned on the $k$th r.v.
the other r.v.'s are independent. The result follows.
\end{proof}
\begin{lemma}\label{th:triang_ineq_for_inner_terms}
Let $\Dist$ be a metric and $\PI$ a mass over $\XU{1}\times \mydots \times\XU{\NumDist}$. Let $\PIU{i,j}$ be the marginal of $\PI$ over $\XU{i}\times \XU{j}$. Define $w_{i,j} = \inner{\DistU{i,j}}{\PIU{i,j}}_{\ell}^{\frac{1}{\ell}}$. For any $i,j,k\in [\NumDist]$ and $\ell \in \mathbb{N}$ we have that $w_{i,j} \leq w_{i,k} + w_{k,j}$.
\end{lemma}
\begin{proof}
Let $\PIU{i,j,k}$ be the marginal of $\PI$ over $\XU{i} \times \XU{j}\times\XU{k}$. Write
\begin{align*}
 w_{i,j} = \left(\inner{\DistU{i,j}}{\PIU{i,j}}_{\ell}\right)^{{1}/{\ell}} = \left(\sum_{s,t,r} (\DistUL{i,j}{s,t})^{\ell} \PIUL{i,j,k}{s,t,r}\right)^{{1}/{\ell}} \leq \left(\sum_{s,t,r} (\DistUL{i,k}{s,r} + \DistUL{k,j}{r,t})^{\ell} \PIUL{i,j,k}{s,t,r}\right)^{{1}/{\ell}}.   
\end{align*}
 Use Minkowski's inequality on a $L_{\ell}$ space with measure $\PIU{i,j,k}$ to bound this by
 \begin{align*}
   \left(\sum_{s,t,r} (\DistUL{i,k}{s,r})^{\ell} \PIUL{i,j,k}{s,t,r}\right)^{{1}/{\ell}} + \left(\sum_{s,t,r} (\DistUL{k,j}{r,t})^{\ell} \PIUL{i,j,k}{s,t,r}\right)^{{1}/{\ell}} = w_{i,k} + w_{k,j}  
 \end{align*}.
\end{proof}
\section{Proof of Theorem \ref{th:n_metric}}\label{app:proof_of_n_metric}

We will need the following hash function in this proof, as well as the useful Lemmas in Appendix \ref{sec:useful_lemmas}.

\subsection{Special hash function} 

\begin{definition}
The map $\HU{\NumDist}$ transforms a tuple $(i,j),  1 \leq i < j\leq \NumDist$, into either $2$, $3$ or $4$ triples according to
\begin{equation}
(i,j) \mapsto \HU{\NumDist}(i,j) = \HUL{\NumDist}{1}(i,j) \oplus \HUL{\NumDist}{2}(i,j),   
\end{equation}
where two tuples (respectively triples) are assumed duplicates iff all of their components agree and
\begin{equation}
\HUL{\NumDist}{1}(i,j) = \left\{
\begin{array}{ll}
      \{(i,\NumDist+1, h(i))\} & \hspace{-0.3cm}, \text{if } j = \NumDist \land i = 1\\
      \{(i,j,h(i)),(j,\NumDist+1,h(i))\} & \hspace{-0.3cm}, \text{if } \text{otherwise}
\end{array}, 
\right. \nonumber
\end{equation}
and
\begin{equation}
\HUL{\NumDist}{2}(i,j) = \left\{
\begin{array}{ll}
      \{(j,\NumDist+1, h(j))\} & \hspace{-0.3cm}, \text{if } i = j -1 \\
      \{(i,j,h(j)),(i,\NumDist+1,h(j))\} & \hspace{-0.3cm}, \text{if } i < j - 1
\end{array}.
\right. \nonumber
\end{equation}
$h(\cdot)$ is also a function of $\NumDist$ but for simplicity we omit it in the notation. $h(\cdot)$ is defined as
\begin{equation} \label{eq:def_h_for_C_one}
h(i) = 1 + ((i - 2) \text{ mod } n).
\end{equation}
\end{definition}

In what follows, the symbol $\oplus$ denotes a list join operation  with \emph{no duplicate removal}, e.g. $\{x,y\}\oplus\{x,z\}=\{x,y,x,z\}$.

\begin{lemma}\label{th:collision_simple_map}
Let $(a,b,c) \in \HU{\NumDist}(i,j)$ for $1 \leq i < j \leq \NumDist$. Then, $1 \leq a < b \leq \NumDist+1$, $ 1 \leq c \leq n$, and $c \notin \{a,b\}$.
Furthermore, the set
\begin{equation}\label{eq:set_H}
\bigoplus_{ 1 \leq  i < j\leq \NumDist} \HU{\NumDist}(i,j)
\end{equation}
has no duplicates.
\end{lemma}
\begin{proof}
The fact that $1\leq a < b \leq \NumDist +1$ and that $ 1 \leq c \leq n$ is immediate. To see that $c\notin\{a,b\}$, we just need to notice that $h(i) \notin \{i,\NumDist+1\}$ for $i\in[\NumDist]$. The fact that $h(i) \neq \NumDist + 1$ follows the range of $h$ being $[\NumDist]$. If we had $h(i) = i$, then we would have $(i-2) \text{ mod } \NumDist = i-1$, which is not possible. To see that  \eqref{eq:set_H} does not have duplicates, we  need to see that, starting from two different tuples, the different expressions that define the triples that go into \eqref{eq:set_H} can never be equal.

Given $1 \leq i < j \leq \NumDist, 1 \leq i' < j' \leq \NumDist, (i,j) \neq (i',j')$ we will show that
\begin{enumerate}
\item $\HU{\NumDist}(i,j)$ does not have duplicates;
\item $\HU{\NumDist}(i,j)$ and $\HU{\NumDist}(i',j')$ do not have overlaps, that is, $\HUL{\NumDist}{1}(i,j)$, $\HUL{\NumDist}{2}(i,j)$, $\HUL{\NumDist}{1}(i',j')$, and $\HUL{\NumDist}{2}(i',j')$ do not have overlaps with each other.
\end{enumerate}

It is obvious that $\HUL{\NumDist}{1}(i,j)$ does not have duplicates and nor does $\HUL{\NumDist}{2}(i,j)$ according to their definitions. Because $i \neq j$, it is also trivial to show that $\HUL{\NumDist}{1}(i,j)$ and $\HUL{\NumDist}{2}(i,j)$ do not have overlaps, based on their definitions. Therefore, the first claim is indeed true. The burden now is to verify the second claim.

We show that the four sets have no overlaps with each other. We show this two sets at a time, there are in total 6 pairs to consider. As an immediate result of the discussion in the first claim, the following four combinations do not have overlaps: $\HUL{\NumDist}{1}(i,j)$ vs. $\HUL{\NumDist}{2}(i,j)$, $\HUL{\NumDist}{1}(i',j')$ vs. $\HUL{\NumDist}{2}(i',j')$, $\HUL{\NumDist}{1}(i,j)$ vs. $\HUL{\NumDist}{1}(i',j')$, $\HUL{\NumDist}{2}(i,j)$ vs. $\HUL{\NumDist}{2}(i',j')$. The two combinations left are $\HUL{\NumDist}{1}(i,j)$ vs $\HUL{\NumDist}{2}(i',j')$ and $\HUL{\NumDist}{1}(i',j')$ vs $\HUL{\NumDist}{2}(i,j)$. We notice that they are symmetric and, because the choice of the tuples $(i,j), (i', j')$ is arbitrary, we only need to show that $\HUL{\NumDist}{1}(i,j)$ and $\HUL{\NumDist}{2}(i',j')$ do not have overlaps, given $(i,j) \neq (i', j')$.

$\HUL{\NumDist}{1}(i,j)$ and $\HUL{\NumDist}{2}(i',j')$ each have two possibilities for the form of their output. Thus, together, there are four possibilities to consider. None of them have an overlap, which we show by contradiction. 
\begin{enumerate}
    \item $\HUL{\NumDist}{1}(i,j) =  \{(i,\NumDist+1,h(i))\}$ and $\HUL{\NumDist}{2}(i',j') = \{(j',\NumDist+1, h(j'))\}$. If these single-element sets have an overlap, that implies that $i=j'$, but, according to the definition, $i=1$ and $i'=j'-1$ which implies $j'>1$.
    \item $\HUL{\NumDist}{1}(i,j) =  \{(i,\NumDist+1,h(i))\}$ and $\HUL{\NumDist}{2}(i',j') = \{(i',j',h(j')), (i',\NumDist+1,h(j'))\}$. For them to have an overlap, $h(i) = h(j')$. That requires $i = j'$ which contradictory to $i=1$ and $i' < j' - 1$ at the same time.
    \item $\HUL{\NumDist}{1}(i,j) =  \{(i,j,h(i)),(j,\NumDist+1,h(i))\}$ and $\HUL{\NumDist}{2}(i',j') = \{(i',j',h(j')), (i',\NumDist+1,h(j'))\}$. For the first two components to equal, $i=i'$, $ j=j'$, and $i=j'$, which is contradictory to $i' < j'-1$. For the second two components to equal, $j=i'$ and $i=j'$, which is contradictory to $i < j$ or $i' < j'$. Because of the existence of ``$n+1$'', the components at different positions cannot collide.
    \item $\HUL{\NumDist}{1}(i,j) =  \{(i,j,h(i)),(j,\NumDist+1,h(i))\}$ and $\HUL{\NumDist}{2}(i',j') =  \{(j',\NumDist+1, h(j'))\}$. This implies $j'=j$ and $j'=i$, which is contradictory to $i < j$.
\end{enumerate}
%
\end{proof}
For example, if $\NumDist = 3$, then the possible tuples $(1,2)$, and $(1,3)$, and $(2,3)$, get mapped respectively to $(1,2,3)$, $(2,4,3)$, $(2,4,1)$, and $(1,4,3)$, $(1,3,2)$, $(1,4,2)$, and $(2,3,1)$, $(3,4,1)$, $(3,4,2)$, all of which are different and satisfy the claims in Lemma \ref{th:collision_simple_map}.

We now prove the four metric properties in order. It is trivial to prove the first three properties given the definition of our distance function for the transport problem. Then, we provide a detailed proof for the triangle inequality.
\subsection{Non-Negativity} 
\begin{proof}
The non-negativity of $\DistU{i,j}$ and $\RIU{i,j}$, implies that 
$\inner{\DistU{i,j}}{\RIU{i,j}}_{\ell} \geq 0$, and hence that $\W \geq 0$.
\end{proof}
\subsection{Symmetry} 
\begin{proof}
Recall that the computation of $\W(\PIU{i_{1:\NumDist}})$ involves a set of distances $\{{\Dist}^{a,b}\}_{a,b}$.
Consider a generic permutation map $\sigma$, and let $\sigma^{-1}$ be its inverse. Let $\sigma$ and $\sigma^{-1}$ apply component-wise to its arguments.
The computation of $\W(\PIU{\sigma(i_{1:\NumDist})})$ involves a set of distances $\{\tilde{\Dist}^{a,b}\}_{a,b}$ that satisfy $\tilde{\Dist}^{i,j} = {\Dist}^{\sigma^{-1}(i,j)}$.
Therefore, each term $\inner{\tilde{\Dist}^{i,j}}{\RIU{i,j}}_{\ell}$ involved in the computation of $\W(\PIU{\sigma(i_{1 : \NumDist})})$, can be rewritten as 
$\inner{{\Dist}^{\sigma^{-1}(i,j)}}{\RIU{i,j}}_{\ell}$, where a simple reindexing of the summation $\sum_{i<j}$ allow us to write as $\inner{{\Dist}^{i,j}}{\RIU{\sigma(i,j)}}_{\ell}$. Since the mass function $\RI$ has as supporting sample space 
$\XU{\sigma(i_1)}\times\mydots\times\XU{\sigma(i_{\NumDist})}$, the marginal $\RIU{\sigma(i,j)}$ can be seen as the marginal $\QIU{i,j}$ of a mass function $\QI$ with support 
$\XU{i_1}\times\mydots\times\XU{i_{\NumDist}}$.
Therefore, minimizing $\sum_{i<j} (\inner{\tilde{\Dist}^{i,j}}{\RIU{i,j}}_{\ell})^{1/\ell}$ for $\RI$ over  
$\XU{\sigma(i_1)}\times\mydots\times\XU{\sigma(i_{\NumDist})}$ is the same as minimizing $\sum_{i<j} (\inner{{\Dist}^{i,j}}{\QIU{i,j}}_{\ell})^{1/\ell}$ for $\QI$ over $\XU{i_1}\times\mydots\times\XU{i_{\NumDist}}$.
\end{proof}
\subsection{Identity} 

\begin{proof}
We prove each direction of the equivalence separately.
Recall that $\{\PIU{i}\}$ are given, they are the masses for which we want to compute the pairwise MMOT.

``$\Longleftarrow$'': If for each $i,j\in [\NumDist]$
we have $\XU{i} = \XU{j}$, then $\SizeDistU{i} = \SizeDistU{j}$, and there exists a bijection $b^{i,j}(\cdot)$ from $[\SizeDistU{i}]$ to $[\SizeDistU{j}]$ such that $\XUL{i}{s} = \XUL{j}{b^{i,j}(s)}$ for all $s$. If furthermore $\PIU{i} = \PIU{j}$, we can define a $\RI$ for $\XU{1}\times\XU{\NumDist}$ such that its univariate marginal over $\XU{i}$, $\RIU{i}$, satisfies $\RIU{i} = \PIU{i}$, and such that its bivariate marginal over $\XU{i}\times\XU{j}$, $\RIU{i,j}$, satisfies $\RIUL{i,j}{s,t} = \PIUL{i}{s}$, if $t  = b^{i,j}(s)$, and zero otherwise. Such a $\RI$ achieves an objective value of $0$ in \eqref{eq:main_def_of_W}, the smallest value possible by the first metric property (already proved). Therefore, $\WU{1,\mydots,\NumDist} = 0$.

``$\Longrightarrow$'': Now let $\RIs$ be a minimizer of \eqref{eq:main_def_of_W} for $\WU{1,\mydots,n}$.
Let $\{\RIsU{i}\}$ and $\{\RIsU{i,j}\}$ be its univariate and bivariate marginals respectively.
If $\WU{1,\mydots,n} = 0$ then $\inner{\DistU{i,j}}{\RIsU{i,j}}_{\ell} = 0$ for all $i,j$. 
Let us consider a specific pair $i,j$, and, without loss of generality, let us assume that $\SizeDistU{i} \leq \SizeDistU{j}$.
Since, by assumption, we have that $\RIsUL{i}{s} = \PIUL{i}{s} > 0$ for all $s\in [\SizeDistU{i}]$, and $\RIsUL{j}{s} = \PIUL{j}{s} > 0$ for all $s\in [\SizeDistU{j}]$, there exists an injection
$b^{i,j}(\cdot)$ from $[\SizeDistU{i}]$ to $[\SizeDistU{j}]$
such that $\RIsUL{i,j}{s,b^{i,j}(s)} > 0$
for all $s\in [\SizeDistU{i}]$. 
Therefore, $\inner{\DistU{i,j}}{\RIsU{i,j}}_{\ell} = 0$ implies that  $\DistUL{i,j}{s,b^{i,j}(s)} = 0$ for all $s\in [\SizeDistU{i}]$.
Therefore, since $\Dist$ is a metric, it must be that $\XUL{i}{s} = \XUL{j}{b^{i,j}(s)}$ for all $s\in [\SizeDistU{i}]$.
Now lets us suppose that there exists an $r\in [\SizeDistU{j}]$ that is not in the range of $b^{i,j}$. Since, by assumption, all of the elements of the sample spaces are different, it must be that $\DistUL{i,j}{s,r} > 0$ for all $s\in [\SizeDistU{i}]$. Therefore, since $\inner{\DistU{i,j}}{\RIsU{i,j}}_{\ell} = 0$, it must be that $\RIsUL{i,j}{s,r} = 0$ for all $s\in [\SizeDistU{i}]$. This contradicts the fact that  $\sum_{s\in [\SizeDistU{i}]} \RIsUL{i,j}{s,r} = \RIsUL{j}{r} = \PIUL{j}{r}> 0$ (the last inequality being true by assumption). Therefore, $\SizeDistU{i} = \SizeDistU{j}$, and the existence of $b^{i,j}$ proves that $\XU{i} = \XU{j}$. At the same time, since $\DistUL{i,j}{s,t} > 0$ for all $t \neq b^{i,j}(s)$, it must be that $\RIsUL{i,j}{s,t} = 0$ for all $t \neq b^{i,j}(s)$. Therefore, $\PIUL{i}{s} = \PIUL{j}{b^{i,j}(s)}$ for all $s$, i.e. $\PIU{i} = \PIU{j}$.
\end{proof}
\subsection{Generalized Triangle Inequality}
\begin{proof}
Let $\PIs$ be a minimizer for (the optimization problem associated with) $\WU{1,\mydots,\NumDist}$, and let $\PIsU{i,j}$ be the marginal induced by $\PIs$ for the sample space $\XU{i}\times \XU{j}$. We would normally use $\RIs$ for this minimizer, but, to avoid confusions between $\RI$ and $r$, we avoid doing so. 
We can write that
\begin{equation}\label{eq:proof_of_th_1_first_step}
\WU{1,\mydots,\NumDist} = \sum_{ 1 \leq i < j \leq \NumDist-1} \inner{\DistU{i,j}}{\PIsU{i,j}}_{\ell}^{\frac{1}{\ell}}.
\end{equation}
For $r\in[\NumDist]$, let $\PI^{(\optsymbol r)}$ be a minimizer for $\WU{1,\mydots,r-1,r+1,\mydots,\NumDist+1}$. We would normally use $\RI^{(\optsymbol r)}$ for this minimizer, but, to avoid confusions between $\RI$ and $r$, we avoid doing so. 
For $i,j\in [\NumDist+1]\backslash\{r\}$, let ${\PI^{(\optsymbol r)}}^{i,j}$ be the marginal of $\PI^{(\optsymbol r)}$ for the sample space $\XU{i}\times\XU{j}$. Recall that since $\PI^{(\optsymbol r)}$ satisfies the constraints in \eqref{eq:main_def_of_W}, its marginal for the sample space $\XU{i}$ is $\PIsU{i}$, which is given in advance.

Let $h(\cdot)$ be the map defined as \eqref{eq:def_h_for_C_one}. 
Define the following mass function for $\XU{1}\times\mydots\times\XU{\NumDist+1}$,
%
%
\begin{equation}\label{eq:proof_of_th_1_p_aux}
\QI = \Gmap{\PIsU{\NumDist+1},\{{\PI^{(\optsymbol h(i))}}^{i \mid \NumDist+1}\}_{i \in [\NumDist]}},
\end{equation}
where ${\PI^{(\optsymbol h(i))}}^{i \mid \NumDist+1}$ is defined as the mass function that satisfies ${\PI^{(\optsymbol h(i))}}^{i \mid \NumDist+1} \PIsU{\NumDist+1} = {\PI^{(\optsymbol h(i))}}^{i, \NumDist+1}$.
Notice that since $h(i) \notin \{i,\NumDist+1\}$, the probability ${\PI^{(\optsymbol h(i))}}^{i, \NumDist+1}$ exists for all $i \in [\NumDist]$.

Let $\QIU{1,\mydots,\NumDist}$ be the marginal of $\QI$ for sample space $\XU{1}\times\mydots\times\XU{\NumDist}$, and $\QIU{i,j}$ be the marginal of $\QI$ for $\XU{i}\times\XU{j}$.

By Lemma \ref{th:marginals_of_G_map}, we know that the $i$th univariate marginal of $\QI$ is $\PIU{i}$ (given) and hence $\QIU{1,\mydots,\NumDist}$ satisfies the constraints associated with $\WU{1,\mydots,\NumDist}$. Therefore, we can write that 
\begin{equation}\label{eq:proof_of_th_1_second_step}
\sum_{ 1 \leq i < j \leq \NumDist} \inner{\DistU{i,j}}{\PIsU{i,j}}_{\ell}^{\frac{1}{\ell}} \leq \sum_{ 1 \leq i < j \leq \NumDist} \inner{\DistU{i,j}}{\QIU{i,j}}_{\ell}^{\frac{1}{\ell}}.
\end{equation}
By Lemma \ref{th:triang_ineq_for_inner_terms}, inequality \eqref{eq:proof_th_1_ineq_1} below holds; because $\Dist$ is symmetric, \eqref{eq:proof_of_th_1_eq_1} below holds; by the definition of $\QI$, \eqref{eq:proof_of_th_1_eq_2} below follows. Therefore,
\begin{align}
\inner{\DistU{i,j}}{\QIU{i,j}}_{\ell}^{\frac{1}{\ell}} &\labelrel\leq{eq:proof_th_1_ineq_1} \inner{\DistU{i,\NumDist+1}}{\QIU{i,\NumDist+1}}_{\ell}^{\frac{1}{\ell}} + \inner{\DistU{\NumDist+1,j}}{\QIU{\NumDist+1,j}}_{\ell}^{\frac{1}{\ell}}\label{}\\
&\labelrel={eq:proof_of_th_1_eq_1} \inner{\DistU{i,\NumDist+1}}{\QIU{i,\NumDist+1}}_{\ell}^{\frac{1}{\ell}} + \inner{\DistU{j,\NumDist+1}}{\QIU{j,\NumDist+1}}_{\ell}^{\frac{1}{\ell}} \nonumber \\
& \labelrel={eq:proof_of_th_1_eq_2} \inner{\DistU{i,\NumDist+1}}{{\PI^{(\optsymbol h(i))}}^{i, \NumDist+1}}_{\ell}^{\frac{1}{\ell}} + \inner{\DistU{j,\NumDist+1}}{{\PI^{(\optsymbol h(j))}}^{j, \NumDist+1}}_{\ell}^{\frac{1}{\ell}}.\label{eq:proof_of_th_1_next_step}
\end{align}

Let $w_{(i,j)}$ denote each term on the r.h.s. of \eqref{eq:proof_of_th_1_first_step}, and $w_{(i,j,r)}$ denote $\inner{\DistU{i,j}}{{\PI^{(\optsymbol r)}}^{i, j}}_{\ell}^{\frac{1}{\ell}}$. 
Combining \eqref{eq:proof_of_th_1_second_step} - \eqref{eq:proof_of_th_1_next_step}, we have
\begin{align}\label{eq:proof_of_th_1_intermediate_ineq}
    \sum_{ 1 \leq i < j \leq \NumDist-1} \hspace{-0.3cm} w_{(i,j)} \leq \hspace{-0.3cm} \sum_{ 1 \leq i < j \leq \NumDist-1} w_{(i,\NumDist,h(i))} + w_{(j,\NumDist,h(j))}.
\end{align}

Finally, we write 
\begin{equation}\label{eq:proof_of_th_1_def_rhs_triangle}
\sum^{\NumDist}_{r = 1} \WU{1,\mydots,r-1,r+1,\mydots,\NumDist+1} =  \sum^{\NumDist}_{r = 1} \sum_{i,j\in [\NumDist+1] \backslash \{r\}, i < j} w_{(i,j,r)},
\end{equation}
and show that \eqref{eq:proof_of_th_1_def_rhs_triangle} upper-bounds the r.h.s of \eqref{eq:proof_of_th_1_intermediate_ineq}.

First, by Lemma \ref{th:triang_ineq_for_inner_terms} and the symmetry of $d$, we have
\begin{equation}
  w_{(i,n,h(i))} \leq  w_{(i,j,h(i))} + w_{(j,n,h(i))}, \label{eq:proof_of_th_1_last_ineq_1}
\end{equation}
and,
\begin{equation}
  w_{(j,n,h(j))} \leq  w_{(i,j,h(j))} + w_{(i,n,h(j))}, \label{eq:proof_of_th_1_last_ineq_2}
\end{equation}
as long as for each triple $(a,b,c)$ in the above expressions, $c \notin \{a,b\}$.
We will use these inequalities to upper bound some of the terms on the r.h.s. of \eqref{eq:proof_of_th_1_intermediate_ineq}, which can be further upper bounded by \eqref{eq:proof_of_th_1_def_rhs_triangle}. 
In particular, we will apply inequalities \eqref{eq:proof_of_th_1_last_ineq_1} and \eqref{eq:proof_of_th_1_last_ineq_2} such that the terms $w_{a,b,c}$ that we get   after their use have triples $(a,b,c)$ that match the triples obtained via the map $\HU{ \NumDist}$ defined in Appendix \ref{sec:useful_lemmas}.
To be concrete, for example, if $\HU{\NumDist}$ maps $(i,j)$ to $\{(i,\NumDist+1,h(i)),(j,\NumDist+1,h(j))\}$, then we do not apply 
\eqref{eq:proof_of_th_1_last_ineq_1} and \eqref{eq:proof_of_th_1_last_ineq_2}, and we leave $w_{(i,\NumDist+1,h(i))} + w_{(j,\NumDist+1,h(j))}$ as is on the r.h.s. of \eqref{eq:proof_of_th_1_intermediate_ineq}. If, for example, $\HU{\NumDist}$ maps $(i,j)$ to $\{(i,\NumDist+1,h(i)),(i,j,h(j)),(i,\NumDist+1,h(j))\}$, then we leave the first term in $w_{(i,\NumDist+1,h(i))} + w_{(j,\NumDist+1,h(j))}$ in the r.h.s. of \eqref{eq:proof_of_th_1_intermediate_ineq} untouched, but we upper bound the second term using \eqref{eq:proof_of_th_1_last_ineq_2} to get $w_{(i,\NumDist+1,h(i))} + w_{(i,j,h(j))} + w_{(j,\NumDist+1,h(j))}$. 

After proceeding in this fashion, and by Lemma \ref{th:collision_simple_map}, we know that all of the terms $w_{(a,b,c)}$ that we obtain  have triples $(a,b,c)$ with $c\neq \{a,b\}$, with $c\in [\NumDist]$, and $1\leq a < b \leq \NumDist+1$. Therefore, these terms appear in \eqref{eq:proof_of_th_1_def_rhs_triangle}. Also by Lemma \ref{th:collision_simple_map}, we know that we do not get each triple more than once. Therefore, the upper bound that we just constructed with the help of $\HU{\NumDist}$ for the r.h.s of \eqref{eq:proof_of_th_1_intermediate_ineq} can be upper bounded by \eqref{eq:proof_of_th_1_def_rhs_triangle}.
\end{proof}
\section{Proof of  Theorem \ref{th:n_metric_c}}\label{app:proof_of_n_metric_C_upper_bound}
The proof of Theorem \ref{th:n_metric_c} requires  a special hash function, discussed next, as well as the useful Lemmas in Appendix \ref{sec:useful_lemmas}.
\subsection{Special hash function}
To prove that \ref{th:n_metric_iv}. in Def. \ref{def:gen_metric_prop_W} holds with $C(\NumDist) \geq (\NumDist-1)/5$, $\NumDist > 7$, we define the following special hash function.
In what follows, the symbol $\oplus$ denotes a list join operation  with \emph{no duplicate removal}, e.g. $\{x,y\}\oplus\{x,z\}=\{x,y,x,z\}$.

\begin{definition}\label{def:complex_map_H_prime}
The map $\HU{'\NumDist}$ transforms a triple $(i,j,r), 1\leq i < j \leq \NumDist, r \in {\color{black}[\NumDist - 1]}$ to \ignore{a list of }either $2$, $3$, or $4$ triples according to 
\begin{equation}
(i,j,r) \mapsto {\color{black}\HU{'\NumDist}(i,j,r)} = \HUL{'\NumDist}{1}(i,j,r) \oplus \HUL{'\NumDist}{1}(j,i,r),
\end{equation}
\begin{equation}
\HUL{'\NumDist}{1}(i,j,r) = \left\{ 
\begin{array}{ll}
      \{(i,r, h'(i,r))\} & , \text{if } j = h'(i,r) \\
      \{(i,j,h'(i,r),(j,r,h'(i,r))\} & , \text{if } j \neq h'(i,r)
\end{array},
\right. 
\end{equation}
\begin{equation}\label{eq:def_h_prime_for_C_n}
h'(i,r) = \left\{
\begin{array}{ll}
    1 + ((i + r - 1) \mod\NumDist) &, \text{if } i < \NumDist\\
    1 + (r \mod (\NumDist - 1))  &, \text{if } i = \NumDist 
\end{array}. 
\right. 
\end{equation}
We assume that the first two components of each output triple are ordered. For example, $(i, r, h'(j,r)) \equiv (\min\{i,r\}, \max\{i,r\}, h'(j,r))$.
\end{definition}

The following property of $\HU{'\NumDist}$ is critical to lower bound $C(\NumDist)$. 

\begin{lemma}\label{th:collision_complex_map}
Let $(a,b,c) \in \HU{'\NumDist}(i,j,r)$, $1 \leq i < j \leq \NumDist$, $r \in [\NumDist-1]$. Then, $1 \leq a \leq b \leq \NumDist$, $ 1 \leq c \leq n$, and $c \notin \{a,b\}$.
Furthermore,
\begin{equation}\label{eq:H_p_set}
\bigoplus_{ 1 \leq  i < j\leq \NumDist} \HU{'\NumDist}(i,j,r)
\end{equation}
has at most $5$ copies of each triple, where two triples are equal iff they agree component-wise.
\end{lemma}
\begin{figure}[h!]
\centering
\includegraphics[width=0.17\textwidth,trim={0.0cm -0.4cm 0cm 0.cm},clip]{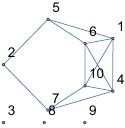}
\caption{\small Graph whose edges are pairs of scenarios that cannot both hold. Any maximum independent set has size $5$, which is used below to prove Lemma \ref{th:collision_complex_map}.}
\label{app:fig:collision_graph}
\end{figure}
\begin{proof}

Recall the definitions:
\begin{equation}
\HUL{'\NumDist}{1}(i,j,r) = \left\{ 
\begin{array}{ll}
      \{(i,r,h'(i,r))\} & , \text{if } j = h'(i,r) \\
      \{(i,j,h'(i,r),(j,r,h'(i,r))\} & , \text{if } j \neq h'(i,r)
\end{array}, 
\right. \nonumber
\end{equation}
\begin{equation}
h'(i,r) = \left\{
\begin{array}{ll}
    1 + ((i + r - 1) \mod \NumDist) &, \text{if } i < \NumDist\\
    1 + (r \mod (\NumDist - 1))  &, \text{if } i = \NumDist 
\end{array}.
\right.  \nonumber
\end{equation}
The fact that $1 \leq a \leq b \leq \NumDist$ and that $ 1 \leq c \leq n$ is immediate. 
The fact that $c \neq \{a,b\}$ amounts to checking that $h'(i,r) \notin  \{i,r\}$ for all $i\in [\NumDist]$, and $r\in [\NumDist-1]$. This can be checked directly from \eqref{eq:def_h_prime_for_C_n}. E.g. $h'(i,r) = i$ would imply either  $(i = \NumDist) \land (r \text{ mod } \NumDist -1 = i-1)$, or $(i < \NumDist) \land ((i+r-1)\text{ mod } \NumDist = i-1$), both of which are impossible. 
The rest of the proof amounts to checking that if $(a,b,c)$ is in the output of $\HU{'\NumDist}$, then there are at most 5 different inputs that lead to $(a,b,c)$.
There are 10 possible candidate input triples that lead to output $(a,b,c)$. Namely,
\begin{enumerate}[]
\item $(a, b, c) = (i_{1}, r_{1}, h'(i_{1}, r_{1})) = \HU{'\NumDist}(i_{1}, j_{1}, r_{1}),$
if $j_{1} = h'(i_{1}, r_{1})$  and $i_{1} < r_{1}$,
\item $(a, b, c) = (r_{2}, i_{2}, h'(i_{2}, r_{2}))=\HU{'\NumDist}(i_{2}, j_{2}, r_{2}),$
if $j_{2} = h'(i_{2}, r_{2})$ and $r_{2} < i_{2}$,
\item $(a, b, c) = (i_{3}, j_{3}, h'(i_{3}, r_{3}))= \HU{'\NumDist} (i_{3}, j_{3}, r_{3}),$
if $j_{3} \neq h'(i_{3}, r_{3})$,
\item $(a, b, c) = (j_{4}, r_{4}, h'(i_{4}, r_{4}))= \HU{'\NumDist}(i_{4}, j_{4}, r_{4}),$
if $j_{4} \neq h'(i_{4}, r_{4})$ and $j_{4} < r_{4}$,
\item $(a, b, c) = (r_{5}, j_{5}, h'(i_{5}, r_{5}))= \HU{'\NumDist}(i_{5}, j_{5}, r_{5}),$
if $j_{5} \neq h'(i_{5}, r_{5})$ and $r_{5} < j_{5}$,
\item $(a, b, c) = (j_{6}, r_{6}, h'(j_{6}, r_{6}))= \HU{'\NumDist}( (i_{6}, j_{6}, r_{6}),$
if $i_{6} = h'(j_{6}, r_{6})$ and $j_{6} < r_{6}$,
\item $(a, b, c) = (r_{7}, j_{7}, h'(j_{7}, r_{7}))= \HU{'\NumDist} (i_{7}, j_{7}, r_{7}),$
if $i_{7} = h'(j_{7}, r_{7})$ and $r_{7} < j_{7}$,
\item $(a, b, c) = (i_{8}, j_{8}, h'(j_{8}, r_{8}))= \HU{'\NumDist} (i_{8}, j_{8}, r_{8}),$
if $i_{8} \neq h'(j_{8}, r_{8})$,
\item $(a, b, c) = (i_{9}, r_{9}, h'(j_{9}, r_{9}))= \HU{'\NumDist} (i_{9}, j_{9}, r_{9}),$
if $i_{9} \neq h'(j_{9}, r_{9})$ and $i_{9} < r_{9}$,
\item $(a, b, c) = (r_{10}, i_{10}, h'(j_{10}, r_{10}))= \HU{'\NumDist} (i_{10}, j_{10}, r_{10}),$
if $i_{10} \neq h'(j_{10}, r_{10})$ and $r_{10} < i_{10}$.
\end{enumerate}
Twelve pairs of the 10 scenarios above cannot simultaneously hold. The 12 pairs of scenarios that cannot both hold are displayed as edges in a graph in Figure \ref{app:fig:collision_graph}. E.g. edge $(4,10)$ represents that scenarios 4 and 10 that cannot both hold. The proof that the pairs represented by edges in Figure \ref{app:fig:collision_graph} cannot both hold is done below.
A maximum independent set of this graph is $\{3,8,9,2,x\}$, where $x \in \{1,4,6,10\}$. Thus at most 5 input scenarios lead to a given output triple.

What remains to be proved is that several pairs of the 10 scenarios described above cannot both hold.

Recall that for any input triple $(i,j,r)$ we always have $1 \leq i < j \leq \NumDist$, and $r\in [\NumDist-1]$.

Scenarios 1 and 5 cannot both hold, because that would imply $r_5=i_1, j_5=r_1$, which would imply $h'(r_5,j_5)=h'(i_1,r_1)=h'(i_5,r_5)$, which since $r_5,i_5<\NumDist$ would imply $i_5=j_5$, contradicting $i_5 < j_5$.

Scenarios 1 and 6 cannot  both hold, because that would imply $a = i_{1} < j_{1} = h'(i_{1}, r_{1}) = c = h'(j_{6}, r_{6}) = i_{6} < j_{6} = a$.

Scenarios 2 and 7 cannot both hold, because that would imply that $j_2 = h'(i_2,r_2) = h'(j_7,r_7) = i_7 < j_7 = i_2$,  contradicting $i_2 < j_2$.

Scenarios 1 and 4 cannot both hold, because that would imply
$j_4=i_1<\NumDist$ and $r_1 = r_4$, which would imply $h'(j_4,r_4)=h'(i_1,r_1)=h'(i_4,r_4)$, which since $i_4,j_4 < \NumDist$ would imply $j_4 = i_4$, contradicting $i_4 < j_4$.

Scenarios 2 and 5 cannot both hold, because that would imply 
$j_5=i_2<\NumDist, r_5=r_2$, which would imply $h'(j_5,r_5)=h'(i_2,r_2)=h'(i_5,r_5)$, which since $j_5<\NumDist$ would imply $i_5=j_5$, contradicting $i_5 < j_5$.

Scenarios 6 and 10 cannot both hold, because that would imply $r_{10} = j_6 < r_6 = i_{10} < \NumDist$, which would imply $h'(r_{10},i_{10})=h'(j_6,r_6)=h'(j_{10},r_{10})$. This in turn would imply one of two things. If $j_{10} < \NumDist$, then $h'(r_{10},i_{10})=h'(j_{10},r_{10})$ would imply $j_{10}=i_{10}$, contradicting $i_{10} < j_{10}$. If on the other hand $j_{10} = \NumDist$, then $h'(r_{10},i_{10})=h'(j_{10},r_{10})$ would imply $1+(r_{10} + i_{10} - 1 \text{ mod } \NumDist) = 1 + (r_{10} \text{ mod } \NumDist -1)$. Recalling that $r_{10} < r_6 \leq \NumDist-1$, we would get $i_{10} - 1 = 0 \text{ mod } \NumDist$. This would imply $i_{10} = 1$, contradicting $i_{10} > r_6 \geq 1$.

Scenarios 7 and 10 cannot both hold, because that would imply $r_7=r_{10}<j_7=i_{10}\leq \NumDist-1$, which would imply $h'(i_{10},r_{10})=h'(j_7,r_7)=h'(j_{10},r_{10})$. This in turn would imply one of two things. If $j_{10} < \NumDist$, then $h'(i_{10},r_{10})=h'(j_{10},r_{10})$ would imply $i_{10}=j_{10}$, contradicting $i_{10} < j_{10}$. If, on the other hand, $j_{10} = \NumDist$, then $h'(i_{10},r_{10})=h'(j_{10},r_{10})$ would imply $1+(i_{10} + r_{10} - 1 \text{ mod } \NumDist) = 1 + (r_{10} \text{ mod } \NumDist -1)$. Recalling that $r_{10} < j_7 =i_{10}\leq \NumDist-1$, we would get $i_{10} - 1 = 0 \text{ mod } \NumDist$. This would imply $i_{10} = 1$, contradicting $i_{10} > r_{10} \geq 1$.

Scenarios 5 and 6 cannot both hold, because that would imply $j_6 = r_5 < j_5 = r_6 \leq \NumDist-1$, and $h'(i_5, r_5)=h'(j_6,r_6)$, which would imply $h'(i_5,j_6) = h'(j_6,j_5)$, which since $j_6 \leq \NumDist-2$ would imply $i_5=j_5$, contradicting $i_5< j_5$.

Scenarios 1 and 10 cannot both hold, because that would imply $r_1=i_{10} > i_1 = r_{10} \geq 1$, which would imply $h'(r_{10},i_{10}) = h'(i_1,r_1) = h'(j_{10},r_{10})$. This would imply one of two things. If $j_{10}< \NumDist$, and, recalling that $r_{10}\leq \NumDist$, this would imply $i_{10} = j_{10}$, contradicting $i_{10} < j_{10}$. If on the other hand, $j_{10} =  \NumDist$, this would imply $1 + (r_{10} \text{ mod } \NumDist - 1) = 1 + (r_{10}  +i_{10} - 1 \text{ mod } \NumDist)$, which would imply $i_{10} = 1$, contradicting $i_{10} > 1$.

Scenarios 4 and 6 cannot both hold, because that would imply $j_4 = j_6 < r_4 = r_6 < \NumDist$, which would imply $h'(i_4,r_4)=h'(j_6,r_6) = h'(j_4,r_4)$, which recalling that $j_4 < \NumDist$ would imply $i_4 = j_4$, contradicting  $i_4 < j_4$.

Scenarios 4 and 7 cannot both hold, because that would imply $j_4 = r_7 < j_7 = r_4 < \NumDist$, which would imply $h'(i_4,r_4)=h'(r_7,j_7) = h'(j_4,r_4)$, which recalling that $j_4 < \NumDist$ would imply $i_4 = j_4$, contradicting  $i_4 < j_4$.

Scenarios 4 and 10 cannot both hold, because that would imply $i_4 < j_4 = r_{10} < r_4 = i_{10} < j_{10}$, which would imply $h'(i_4,i_{10}) = h'(i_4,r_{4}) = h'(j_{10},r_{10}) =  h'(j_{10},j_{4})$. This would imply one of two things. If $j_4 < \NumDist$, this would imply $1 + ( i_4 + i_{10} + 1\text{ mod } \NumDist) = 1 + ( j_4 + j_{10} + 1\text{ mod } \NumDist)$, which would imply $(j_4 - i_4)+ (j_{10} - i_{10}) = 0 \text{ mod } \NumDist$, which since $j_4 > i_4,j_{10} > i_{10}$ would imply $(j_4 - i_4)+ (j_{10} - i_{10}) = \NumDist$. This in turn would imply $j_{10}= \NumDist + (i_{10} - j_4)+ i_4 > \NumDist + i_4 > \NumDist$, since $i_{10} - j_4 > 0$, and $i_4 > 0$, contradicting $j_{10} \leq \NumDist$. If on the other hand $j_4 = \NumDist$, this would imply $1 + ( i_4 + i_{10} + 1\text{ mod } \NumDist) = 1 + ( j_4 \text{ mod } \NumDist - 1)$, which since $j_4  = r_{10} < r_4 \leq \NumDist -1$ would imply $(j_4 - i_4 - i_{10} + 1) \text{ mod \NumDist} = 0$, which imply either $j_4 - i_4 - i_{10} + 1 = 0$ or $j_4 \hspace{-0.1cm}-\hspace{-0.1cm} i_4 \hspace{-0.1cm}-\hspace{-0.1cm} i_{10} + 1  = -\NumDist$. The 1st option would imply $j_4 = i_{10} \hspace{-0.1cm}+\hspace{-0.1cm} i_4 \hspace{-0.1cm}-\hspace{-0.1cm} 1 \hspace{-0.1cm}\geq\hspace{-0.1cm} i_{10}$, contradicting $j_4 \hspace{-0.1cm}<\hspace{-0.1cm} i_{10}$. The 2nd option would imply $i_{10} \hspace{-0.1cm}=\hspace{-0.1cm} \NumDist \hspace{-0.1cm}+\hspace{-0.1cm} 1 \hspace{-0.1cm}+\hspace{-0.1cm} j_4 \hspace{-0.1cm}-\hspace{-0.1cm} i_4$ \hspace{-0.1cm}>\hspace{-0.1cm} \NumDist, contradicting $i_{10} \hspace{-0.1cm}< \NumDist$.
\end{proof}
\begin{remark}
Note that we might have $a = b$ in an triple $(a,b,c)$ output by $\HU{'\NumDist}$. For example, if $n = 4$, all $5$ triples $(1,2,3)$, $(1,3,2)$, $(2,3,2)$, $(2,3,3)$, and $(2,3,4)$ map to $(2,3,1)$. Also, both $(1,2,1)$ and $(1,4,1)$ map to $(1,1,2)$ whose first two components equal.
\end{remark}
\subsection{Proof of lower bound on $C(n)$}
We will show 
\begin{align*}
    (\NumDist-1) \WU{1,\mydots,\NumDist} \leq 
5\sum^{\NumDist}_{r = 1} \WU{1,\mydots,r-1,r+1,\mydots,\NumDist+1}.
\end{align*}
For $r\in[\NumDist]$, let $\PI^{(\optsymbol r)}$ be a minimizer that leads to $\WU{1,\mydots,r-1,r+1,\mydots,\NumDist+1}$. We would normally use $\RI^{(\optsymbol r)}$ for this minimizer, but, to avoid confusions between $\RI$ and $r$, we avoid doing so. 
For $i,j\in [\NumDist+1]\backslash\{r\}$, let ${\PI^{(\optsymbol r)}}^{i,j}$ be the marginal of $\PI^{(\optsymbol r)}$ for the sample space $\XU{i}\times\XU{j}$. Since $\PI^{(\optsymbol r)}$ satisfies the constraints in \eqref{eq:main_def_of_W}, its marginal over $\XU{i}$ equals $\PIU{i}$.

Let $h'(\cdot,\cdot)$ be the map in \eqref{eq:def_h_prime_for_C_n}. 
For each $r \in [\NumDist-1]$, define the mass function over $\XU{1}\times\mydots\times\XU{\NumDist}$
\begin{equation} 
\QI^{(r)} = \Gmap{\PIU{r},\{{\PI^{(\optsymbol h'(i,r))}}^{i \mid r}\}_{i \in [\NumDist] \backslash r}}, \end{equation}
where ${\PI^{(\optsymbol h'(i,r))}}^{i \mid r}$ satisfies ${\PI^{(\optsymbol h'(i,r))}}^{i \mid r} \PIU{r} = {\PI^{(\optsymbol h'(i,r))}}^{i, r}$.
Note that $h'(i,r) \notin \{i,r\}, \forall\ 1\leq i\leq \NumDist$, and $r \in [\NumDist-1]$. Thus, ${\PI^{(\optsymbol h'(i,r))}}^{i, r}$ and ${\PI^{(\optsymbol h'(i,r))}}^{i \mid r}$ exist. Let ${\QIU{(r)}}^{i}$ be the marginal of $\QIU{(r)}$ over $\XU{i}$, and ${\QIU{(r)}}^{i,j}$ over $\XU{i} \times \XU{j}$.

By Lemma \ref{th:marginals_of_G_map}, we know that ${\QIU{(r)}}^{i}$ equals $\PIU{i}$ (given) for all $i \in [\NumDist]$, and hence $\QIU{(r)}$ satisfies the optimization constraints in \eqref{eq:main_def_of_W} for $\WU{1,\mydots,\NumDist}$. 
Therefore, we can write 
\begin{align}\label{eq:proof_of_th_2_second_step}
(\NumDist-1) \WU{1,\mydots,\NumDist}=\sum^{\NumDist-1}_{r = 1} \sum_{ 1 \leq i < j \leq \NumDist} \inner{\DistU{i,j}}{\PIsU{i,j}}_{\ell}^{\frac{1}{\ell}} \leq \sum^{\NumDist-1}_{r = 1} \sum_{ 1 \leq i < j \leq \NumDist} \inner{\DistU{i,j}}{{\QIU{(r)}}^{i,j}}_{\ell}^{\frac{1}{\ell}}, 
\end{align}
where $\PIsU{i,j}$ is the bivariate marginal over $\XU{i}\times\XU{j}$ of the minimizer $\PIs$ for $\WU{1,\mydots,\NumDist}$.

We now bound each term in the inner most sum on the r.h.s. of \eqref{eq:proof_of_th_2_second_step} as 
\begin{align}
\inner{\DistU{i,j}}{{\QIU{(r)}}^{i,j}}_{\ell}^{\frac{1}{\ell}} 
&\labelrel\leq{eq:proof_th_2_ineq_1} \inner{\DistU{i,r}}{{\QIU{(r)}}^{i,r}}_{\ell}^{\frac{1}{\ell}}
+ \inner{\DistU{r,j}}{{\QIU{(r)}}^{r,j}}_{\ell}^{\frac{1}{\ell}}\label{eq:proof_of_th_2_eq_tri_1} \\
&\labelrel={eq:proof_of_th_2_eq_1} \inner{\DistU{i,r}}{{\QIU{(r)}}^{i,r}}_{\ell}^{\frac{1}{\ell}} + \inner{\DistU{j,r}}{{\QIU{(r)}}^{j,r}}_{\ell}^{\frac{1}{\ell}}\\
&\labelrel={eq:proof_of_th_2_eq_2} \inner{\DistU{i,r}}{{\PI^{(\optsymbol h'(i,r))}}^{i, r}}_{\ell}^{\frac{1}{\ell}} + \inner{\DistU{j,r}}{{\PI^{(\optsymbol h'(j,r))}}^{j, r}}_{\ell}^{\frac{1}{\ell}},\label{eq:proof_of_th_2_eq_tri_3}
\end{align}
where $i\neq r$,  $r\neq j$, and: \eqref{eq:proof_th_2_ineq_1}
holds by Lemma \ref{th:triang_ineq_for_inner_terms}; \eqref{eq:proof_of_th_2_eq_1} holds because $\Dist$ is symmetric; and \eqref{eq:proof_of_th_2_eq_2} holds because, by Lemma \ref{th:marginals_of_G_map}, ${\QIU{(r)}}^{i,r}={\PI^{(\optsymbol h'(i,r))}}^{i, r}$ and 
${\QIU{(r)}}^{j,r}={\PI^{(\optsymbol h'(j,r))}}^{j, r}$.

Bounding the r.h.s. of \eqref{eq:proof_of_th_2_second_step} using \eqref{eq:proof_of_th_2_eq_tri_1} - \eqref{eq:proof_of_th_2_eq_tri_3}, we re-write the resulting inequality using the notation
\begin{align}\label{eq:proof_of_th_2_intermediate_ineq}
     (\NumDist-1) \WU{1,\mydots,\NumDist} = \sum^{\NumDist-1}_{r=1}\sum_{ 1 \leq i < j \leq \NumDist} w_{(i,j,r)} 
     \leq  \sum^{\NumDist-1}_{r=1} \sum_{ 1 \leq i < j \leq \NumDist} v_{(i,r,h'(i,r))} + v_{(j,r,h'(j,r))}, 
\end{align}
where (a) each $w_{(i,j,r)}$ represents one $\inner{\DistU{i,j}}{\PIsU{i,j}}_{\ell}^{\frac{1}{\ell}}$ on the l.h.s. of \eqref{eq:proof_of_th_2_second_step}. Since $h'(i,r) \notin \{i,r\}$, when $i \neq r$ the mass ${\PI^{(\optsymbol h'(i,r))}}^{i,r}$ exists; (b) each $v_{(s,t,l)}$ represents  $\inner{\DistU{s,t}}{{\PI^{(\optsymbol l)}}^{s,t}}_{\ell}^{\frac{1}{\ell}}$ if $s \neq t$, and is zero if $s = t$; and (c) we are implicitly assuming that the first two components of each triple on the r.h.s. of \eqref{eq:proof_of_th_2_intermediate_ineq} are ordered, i.e. if e.g. $r < i$ then $(r,i,h'(i,r))$ should be red as $(i,r,h'(i,r))$.

Finally, using this same compact notation, we write 
%
%
\begin{equation}\label{eq:proof_of_th_2_def_rhs_triangle}
5\sum^{\NumDist}_{r = 1} \WU{1,\mydots,r-1,r+1,\mydots,\NumDist+1} = 5 \sum^{\NumDist}_{r = 1} \sum_{i,j\in [\NumDist+1] \backslash \{r\}, i < j} v_{(i,j,r)},
\end{equation}
and now we will show that \eqref{eq:proof_of_th_2_def_rhs_triangle} upper-bounds the r.h.s. of \eqref{eq:proof_of_th_2_intermediate_ineq}, finishing the proof.

First, by Lemma \ref{th:triang_ineq_for_inner_terms} and the symmetry of $d$, observe that the following inequalities are true
\begin{align}
v_{(i,r,h'(i,r))} &\leq  v_{(i,j,h'(i,r))} + v_{(j,r,h'(i,r))}, \label{eq:proof_of_th_2_last_ineq_1} \end{align}
and,
\begin{align}
v_{(j,r,h'(j,r))} &\leq  v_{(i,j,h'(j,r))} + v_{(i,r,h'(j,r))},\label{eq:proof_of_th_2_last_ineq_2}
\end{align}
as long as for each triple $(a,b,c)$ in the above expressions, $c \notin \{a,b\}$.
We will use inequalities \eqref{eq:proof_of_th_2_last_ineq_1} and \eqref{eq:proof_of_th_2_last_ineq_2} to upper bound some of the terms on the r.h.s. of \eqref{eq:proof_of_th_2_intermediate_ineq}, and then we will show that the resulting sum can be upper bounded by \eqref{eq:proof_of_th_2_def_rhs_triangle}.
In particular, for each $(i,j,r)$ considered in the r.h.s. of \eqref{eq:proof_of_th_2_intermediate_ineq}, we will apply inequalities \eqref{eq:proof_of_th_2_last_ineq_1} and \eqref{eq:proof_of_th_2_last_ineq_2} such that  the terms $v_{(a,b,c)}$ that we get after their use have triples $(a,b,c)$ that match the triples in ${\HU{'}}^{\NumDist}(i,j,r)$, defined in Def. \ref{def:complex_map_H_prime}.
To be concrete, for example, if ${\HU{'}}^{\NumDist}$ maps $(i,j,r)$ to $\{(i,r,h'(i,r)),(r,j,h'(j,r))\}$, then we do not apply 
\eqref{eq:proof_of_th_2_last_ineq_1} and \eqref{eq:proof_of_th_2_last_ineq_2}, and we leave $v_{(i,r,h'(i,r))} + v_{(r,j,h'(j,r))}$ as is on the r.h.s. of \eqref{eq:proof_of_th_2_intermediate_ineq}. If, for example, ${\HU{'}}^{\NumDist}$ maps $(i,j,r)$ to $\{(i,r,h'(i,r)),(i,j,h'(j,r)),(i,r,h'(j,r))\}$, then we leave the first term in $v_{(i,r,h'(i,r))} + v_{(r,j,h'(j,r))}$ in the r.h.s. of \eqref{eq:proof_of_th_2_intermediate_ineq} untouched, but we upper bound the second term using \eqref{eq:proof_of_th_2_last_ineq_2} to get $v_{(i,r,h'(i,r))} + v_{(i,j,h'(j,r))} + v_{(i,r,h'(j,r))}$. 

After proceeding in this fashion, and by Lemma \ref{th:collision_complex_map}, we know that all of the terms $v_{(a,b,c)}$ that we obtain have triples $(a,b,c)$ with $c\neq \{a,b\}$, $c\in [\NumDist-1]$, and $1\leq a \leq b \leq \NumDist$. Therefore, these terms are either zero (if $a=b$) or appear in \eqref{eq:proof_of_th_2_def_rhs_triangle}. Also because of Lemma \ref{th:collision_complex_map}, each triple $(a,b,c)$ with non-zero $v_{(a,b,c)}$ will not appear more than $5$ times. Therefore, the upper bound we build with the help of $h'$ for the r.h.s of \eqref{eq:proof_of_th_2_intermediate_ineq} can be upper bounded by \eqref{eq:proof_of_th_2_def_rhs_triangle}.

\subsection{Proof of upper bound on $C(n)$}

Consider the following setup. Let $\SizeDistU{i} = \SizeDist$ for all $i \in [\NumDist]$, and $\XUL{i}{s} \in \mathbb{R}$ for all $i \in [\NumDist]$, $s \in [\SizeDist]$. Define $\Dist$ such that $\DistUL{i,j}{s,t}$ is $\vert\XUL{i}{s} - \XUL{j}{t}\vert$, if $s = t$, and infinity otherwise. Let $\PIUL{i}{s} = \frac{1}{\SizeDist}$ for all  $i \in [\NumDist]$, $s \in [\SizeDist]$. 

Any optimal solution $\RIs$ to the pairwise MMOT problem must have bivariate marginals that satisfy $\RIsUL{i,j}{s,t} = \frac{1}{\SizeDist} \delta_{s,t}$, and thus
$\inner{\DistU{i,j}}{\RIsU{i,j}}_{\ell}^{\frac{1}{\ell}}= \frac{1}{\SizeDist^\ell}\| \XU{i} - \XU{j}\|_{\ell}$,
where we interpret $\XU{i}$ has a vector in $\mathbb{R}^{\SizeDist}$, and $\|\cdot\|_{\ell}$ is the vector $\ell$-norm.
Therefore, ignoring the factor $\frac{1}{\SizeDist^\ell}$, we only need to prove that \ref{th:n_metric_iv}. in Def. \ref{def:gen_metric_prop_W} holds with $C(\NumDist) = \NumDist - 1$ when 
$\WU{1:\NumDist}$ is defined as $\sum_{1\leq i<j\leq \NumDist} \|\XU{i} - \XU{j}\|_{\ell}$.
This in turn is a standard result, whose proof (in a more general form) can be found e.g. in Example 2.4 in \cite{kiss2018generalization}.

\end{appendices}

\end{document}